\pdfoutput=1
\documentclass[letterpaper,twocolumn,10pt]{article}
\usepackage{usenix-2020-09}

\usepackage{tikz}

\usepackage{cite}
\usepackage{amsmath,amssymb,amsfonts}
\usepackage{algorithmic}
\usepackage{graphicx}
\usepackage{textcomp}
\usepackage{xcolor}
\usepackage{mathtools}
\usepackage{amsthm}
\usepackage{dsfont}

\usepackage{hyperref}       
\usepackage{booktabs}       
\usepackage{nicefrac}       
\usepackage{microtype}      
\usepackage{xcolor}         
\usepackage{graphicx}
\usepackage{subfigure}
\usepackage{enumerate}
\usepackage{multirow}

\newtheorem{theorem}{Theorem}

\newtheorem{lemma}{Lemma}

\def\mathE{{\mathbb{E}}}
\def\mathN{{\mathcal{N}}}
\def\mathA{{\mathcal{A}}}

\def\vrho{{\bm{\rho}}}

\def\valpha{{\bm{\alpha}}}

\def\vxi{\bm{\xi}}
\def\vbeta{\bm{\beta}}

\DeclareMathOperator{\tr}{\mathrm{tr}}

\DeclareMathOperator{\diag}{\mathrm{diag}}

\usepackage{amsmath,amsfonts,bm}

\def\eqref#1{equation~\ref{#1}}

\def\1{\bm{1}}

\def\rmI{{\mathbf{I}}}

\def\vmu{{\bm{\mu}}}
\def\vxi{{\bm{\xi}}}

\def\vg{{\bm{g}}}

\def\vw{{\bm{w}}}
\def\vx{{\bm{x}}}

\def\mA{{\bm{A}}}
\def\mB{{\bm{B}}}

\def\mD{{\bm{D}}}

\def\mI{{\bm{I}}}

\def\mP{{\bm{P}}}
\def\mQ{{\bm{Q}}}
\def\mR{{\bm{R}}}

\def\mV{{\bm{V}}}
\def\mW{{\bm{W}}}
\def\mX{{\bm{X}}}
\def\mY{{\bm{Y}}}

\def\mSigma{{\bm{\Sigma}}}
\def\mPsi{{\bm{\Psi}}}

\DeclareMathAlphabet{\mathsfit}{\encodingdefault}{\sfdefault}{m}{sl}
\SetMathAlphabet{\mathsfit}{bold}{\encodingdefault}{\sfdefault}{bx}{n}

\usepackage{marvosym}

\newcommand{\qi}[1]{\textcolor{black}{#1}}

\begin{document}

\date{}

\title{Defending Against Data Reconstruction Attacks in Federated Learning: \\ An Information Theory Approach}

\author{
{\rm Qi Tan\textsuperscript{a},
Qi Li\textsuperscript{b},
Yi Zhao\textsuperscript{c, \Letter},
Zhuotao Liu\textsuperscript{b},
Xiaobing Guo\textsuperscript{d},
and Ke Xu\textsuperscript{a, \Letter}}\\
\\
\textsuperscript{a}Department of Computer Science and Technology, Tsinghua University\\
\textsuperscript{b}Institute for Network Science and Cyberspace, Tsinghua University\\
\textsuperscript{c}School of Cyberspace Science and Technology, Beijing Institute of Technology\\
\textsuperscript{d}Lenovo Research\\
} 

\maketitle

\begin{abstract}
Federated Learning (FL) trains a black-box and high-dimensional model among different clients by exchanging parameters instead of direct data sharing, which mitigates the privacy leak incurred by machine learning. However, FL still suffers from membership inference attacks (MIA) or data reconstruction attacks (DRA). In particular, an attacker can extract the information from local datasets by constructing DRA, which cannot be effectively throttled by existing techniques, e.g., Differential Privacy (DP).

In this paper, we aim to ensure a strong privacy guarantee for FL under DRA. \qi{We prove that reconstruction errors under DRA are constrained by the information acquired by an attacker, which means that constraining the transmitted information can effectively throttle DRA. To quantify the information leakage incurred by FL, we establish a channel model, which depends on the upper bound of joint mutual information between the local dataset and multiple transmitted parameters. Moreover, the channel model indicates that the transmitted information can be constrained through data space operation, which can improve training efficiency and the model accuracy under constrained information. According to the channel model,
we propose algorithms to constrain the information transmitted in a single round of local training. With a limited number of training rounds, the algorithms ensure that the total amount of transmitted information is limited.
Furthermore, our channel model can be applied to various privacy-enhancing techniques (such as DP) to enhance privacy guarantees against DRA. Extensive experiments with real-world datasets validate the effectiveness of our methods.}
\end{abstract}

\section{Introduction}
Federated learning (FL)~\cite{DBLP:conf/aistats/McMahanMRHA17,DBLP:journals/tist/YangLCT19,DBLP:journals/network/ZhaoXCT22,10355680}  is a new form of machine learning (ML),
which protects privacy by transmitting gradients or parameters to avoid sharing raw data. Specifically, the parameters form a \emph{communication channel} between the server and each client, so the server gets information from local datasets via such channels. Based on the Data Processing Inequality (DPI) \cite{DBLP:books/daglib/0016881}, communication by parameters, which is a deterministic mapping of local data, instead of raw data, reduces the risk of data privacy.
However, recent studies reveal that the parameter channel of FL still leaks privacy.
For example, multiple literature indicates that adversaries can conduct membership inference attacks (MIA) with uploaded model parameters~\cite{DBLP:conf/icml/SablayrollesDSO19,DBLP:conf/sp/NasrSH19,DBLP:conf/sp/MelisSCS19,DBLP:conf/eurosp/LongWBB0TGC20,DBLP:conf/icml/Choquette-ChooT21,DBLP:conf/sp/CarliniCN0TT22}, which breaks the anonymity of data privacy. Moreover, adversaries can completely steal training data by data reconstruction attacks (DRA)~\cite{DBLP:conf/ccs/HitajAP17,DBLP:conf/nips/ZhuLH19,DBLP:conf/uss/CarliniTWJHLRBS21,DBLP:conf/nips/HaimVYSI22,DBLP:journals/tifs/ChenZLFX23}, resulting in serious privacy issues in FL.

\qi{In order to enhance privacy protection for FL, dimension reduction~\cite{DBLP:conf/kdd/LiDYC020, DBLP:journals/tkde/OsiaTSKHR20,DBLP:conf/aaai/ThapaCCS22} or differential privacy (DP) \cite{DBLP:conf/nips/AgarwalSYKM18,DBLP:conf/sosp/LecuyerSVG019,DBLP:conf/nips/LevySAKKMS21} are widely adopted approaches.
However, dimension reduction lacks theoretical guarantees of the defense ability against DRA, so it cannot flexibly configure defense capabilities according to different privacy requirements.
DP's privacy protection can provide theoretical guarantees (e.g., the privacy budget $\epsilon$) for MIA attacks~\cite{DBLP:conf/sp/BalleCH22,DBLP:conf/icml/GuoKCM22}. It aims to guarantee that changing any data point will not significantly affect the output distribution of a system. This goal is different from the one in defending against DRA, which focuses on preventing the attacker from reconstructing the whole distribution of the local dataset. Thus DP still cannot defend against DRA attacks~\cite{DBLP:conf/kdd/Cormode11, DBLP:conf/uss/LiuWH000CF022,DBLP:journals/corr/abs-2302-07225,DBLP:conf/sp/BalleCH22,DBLP:conf/icml/GuoKCM22}.
For instance, in DP-SGD, algorithms with identical privacy budget but different training hyper-parameters (e.g., different batch size $B$) cannot guarantee the same success rate for DRA~\cite{DBLP:journals/corr/abs-2302-07225}.
Moreover, quantifying information leakage is the basis for defending against DRA attacks.
Previous technique like Quantitative Information Flow (QIF)\footnote{QIF focuses on a special security concern, namely the probability of guessing a secret in one try~\cite{DBLP:series/isc/AlvimCMMPS20,DBLP:conf/fossacs/Smith09}. This security concern is different from the one in DRA, where DRA focuses on reconstructing the whole distribution of local data.} \cite{DBLP:series/isc/AlvimCMMPS20} quantifies information leak under a white-box and time-invariant setting\cite{DBLP:conf/ccs/0002CPP20}. It requires the knowledge of the correlation between the distributions of inputs and outputs, which cannot hold in FL systems.}

\qi{It is difficult to 
defend against DRA attacks in FL due to the following challenges. (i) \textbf{the black-box model.} For DNN-based models, the correlation between the input distribution and output distribution is extremely complex, and we cannot obtain the exact mapping function between the two distributions, making theoretical analysis impossible.
(ii) \textbf{the high-dimensional parameter space.} Traditional mathematical tools (e.g., eigen-decomposition) cannot process high-dimensional parameter spaces on the scale of thousands of millions due to the requirements for large-scale storage and high-performance computing.
(iii) \textbf{the time-variant system.} During the training process, constant parameter updates change the model at each step, leading to a time-variant system in FL. The time-variant system, which
changes the output distributions accordingly,
leads to dynamic information leakage, requiring continuously changing quantifications.}

\qi{To address the above challenges, we develop a theoretical framework based on mutual information (MI)\footnote{Mutual information $I(X;Y)$ \cite{DBLP:conf/ccs/CuffY16,DBLP:journals/corr/Shwartz-ZivT17,DBLP:conf/iclr/HjelmFLGBTB19} represents the uncertainty decrement of $X$ when we observe $Y$.} to evaluate privacy leakage caused by DRA in FL, and design methods to constrain information leakage to defend against DRA attacks. Specifically, we demonstrate that the lower bound of mean squared error (MSE), which serves as an indicator of DRA's precision in reconstruction (i.e., the smaller MSE means the higher precision for the attacker), is determined by the amount of acquired information, i.e., the MI between the local dataset and the shared parameters.
Thus, MI can be utilized as the indicator for quantifying the information leakage in FL.
Then we build a channel model to analyze information leaks under the black-box setting of FL.
Through our proposed channel model, we find
that the transmitted information (i.e., the information leakage) is decided by two factors: the channel capacity $C$, which represents the maximal ability to transmit information in a single training round; and the optimization rounds $n$, which is correlated to the information accumulation.
For example, if the channel capacity is bounded by a threshold $\kappa$, and the number of optimization rounds is less than $n$, then the total amount of information leakage is less than $n \cdot \kappa$. Furthermore, our channel model can
analyze various privacy-enhancing methods in defending against DRA, including DP, gradient compression, and utilizing large batch size.}


\qi{Based on the channel model, we utilize DPI to transform the operations (e.g., eigen-decomposition and adding noise) of constraining channel capacity from the parameter space to the data space. This transformation significantly improves the training efficiency and the model accuracy of the high dimensional and time-variant model under constrained information leakage.}
Specifically, our \emph{protecting goal} is to decide the covariance matrix for the added noise according to a given data distribution $\mD$, which ensures privacy protection by constraining the reconstruction error above a certain threshold. Compared to conventional protection techniques, which directly deal with parameters after gradient mapping, constraining in the data space has two distinct advantages:
firstly, it makes the computational complexity independent of the optimization rounds $n$ and the model's dimensionality $d_m$, reducing it from $\mathcal{O}(n\cdot d_m)$ to $\mathcal{O}(d_D)$, where $d_D$ denotes the dimensionality of the data and $d_D \ll d_m$. Secondly, data space preserves the correlations between data attributes, hence we can leverage the prior knowledge of relative importance to implement stronger safeguards for the critical attributes, which enhances the capability to balance the utility and the privacy. 

Finally, according to the theoretical results, we propose three implementations for constraining the channel capacity. These implementations incorporate different prior knowledge in defending against DRA, which can be employed to flexibly balance the utility and the privacy.

In summary, the contributions of our paper are as follows:
\begin{itemize}
\setlength{\itemsep}{2pt}
\setlength{\parsep}{0.5pt}
\setlength{\parskip}{0pt}
  \item We demonstrate that the amount of transmitted information decides the lower bound of the reconstruction error for DRA attacks.
  \item \qi{We establish a channel model to quantify the information leakage of the black-box model in FL, which can be applied to analyze various privacy-enhancing methods for defending against DRA.}
  \item \qi{We theoretically constrain the transmitted information through the operation in data space instead of parameter space for the first time and demonstrate that it significantly improves the training efficiency and the model accuracy under constrained information leakage.}
  \item \qi{By incorporating different prior knowledge, we propose three implementations to constrain channel capacity, which can be utilized to flexibly balance the utility and the privacy.}
  \item \qi{Extensive experiments demonstrate that the newly proposed methods effectively enhance the safety, efficiency, and flexibility of FL.}
\end{itemize}

\section{Background and Preliminary}
This paper studies the FL problem based on information theory. Specifically, in the FL scenario, the server and clients communicate by sending model parameters. Even without direct data sharing, the MI between shared parameters and the local dataset grows accordingly, which enhances the ability for attackers to conduct DRA attacks. For clarity purposes, Tab.~\ref{tb:main_notations} lists major notations used in the paper, and we will describe the remaining variables when they are utilized.

\begin{table}[htb]
\vskip -0.15in
\footnotesize
\caption{Major Notation Explanation}\label{tb:main_notations}
\begin{center}
\begin{tabular}{cl}
\toprule
Notations & Explanation\\
\midrule
$F(\cdot)$, $\eta$            &  Loss function and learning rate for local optimization \\
$\mW^{(t)}_{i}$, $\mW^{(t)}_{o}$           &  The received (input) and shared (output) parameters of\\
&the victim at time $t$ \\
$\mD$                       &  Random variable follows data distribution of the victim\\
$B$                         &  Batch size for optimization \\
$E$, $n$                    &  \qi{Local steps for one communication round and  the total}\\ &\qi{local steps for all communication rounds} \\
$\mathcal{A}^{(t)}(\cdot)$  &  The aggregation method at time $t$ \\
$\mV^{(t)}$                 &  Variables (e.g., gradients, parameters) collected by the \\
&server from clients other than the victim \\
$C^{(t)}$       &  The channel capacity (maximum transmitted informati-\\
                &  on) of the victim at time $t$\\
$\kappa$      &  The threshold of $C^{(t)}$, i.e., the setted channel capacity        \\
$\mSigma_{*}$               &  The covariance matrix of random variable \\
$\lambda$ and $\sigma$ & Eigenvalues for the covariance matrix and noise variables \\
$I(\mX; \mY)$        & \qi{The mutual information between $\mX$ and $\mY$} \\
$\bm{\xi}$        & \qi{The noise variable that is subject to Gaussian distribution}\\
\bottomrule
\end{tabular}
\end{center}
\vskip -0.1in
\end{table}

\subsection{Federated Learning}
Regarding FL, 
a specific client, namely the victim, receives the initial parameter $\mW_1$ from the server in a specific communication round and conducts the optimization process with the victim's dataset as
\begin{equation}\label{eq:local_sgd_FL}
  \mW_{t+1} \leftarrow \mW_t - \eta \cdot \nabla_{\mW}F(\mW_t; \mD), \; t = 1,\, \cdots,\, E,
\end{equation}
where $E$ is the number of local steps.
Then the victim sends $\mW_{E+1}$ back to the server. If we rewrite $\mW_1$ and $\mW_{E+1}$ as $\mW_i$ and $\mW_o$ respectively, the two parameters form a communication channel for information transmission (as illustrated in Fig. \ref{fig:unfoldinfoflow}), and the local optimization process defined in Eq.~(\ref{eq:local_sgd_FL}) loads information from the dataset to the communication channel. Finally, the server collects the parameters from different clients (including the victim) for aggregation as
\begin{equation}\label{eq:aggregation_FL}
  \mW_i = \mathA(\mW_o; \mV).
\end{equation}

\subsection{Information Theory}
\noindent {\bfseries Differential Entropy.} The differential entropy of a random variable $\mX$ is defined as follows
\begin{equation}\label{eq:diff_entropy}
  h(\mX) = -\int_{\mX}f(\vx)\log f(\vx)\,d\vx,
\end{equation}
which is utilized to describe the degree of random uncertainty. For a dataset with high information entropy, i.e., a dataset with plentiful information, it is difficult for an attacker to conduct DRA attacks. While for the dataset with low information entropy, the opposite is true.

However, calculating the differential entropy by Eq.~(\ref{eq:diff_entropy}) is infeasible in practice since we cannot obtain the distribution function of the target variable. Therefore, we use the maximum entropy distribution to analyze the worst-case scenario. Specifically, with $\mathE[\mX]=\vmu$ and $Cov(\mX)=\mSigma$, the maximum entropy distribution is the Gaussian distribution, i.e., $h(\mX) \leq h(\mathN(\vmu, \mSigma))$.

\noindent {\bfseries Mutual Information.} For two random variables $\mX$ and $\mY$, the mutual information between them is
\begin{equation}\label{eq:mutual_info}
  I(\mX;\mY) = h(\mX) - h(\mX| \mY) = h(\mY) - h(\mY | \mX).
\end{equation}
Specifically, the mutual information $I(\mX;\mY)$ is a symmetric function, which describes the random uncertainty decrement of $\mX$ when we observe $\mY$, and vice versa.

Specifically, in this paper, we utilize $I(\mD; \mW_i, \mW_o)$ to quantify the information leakage in FL. The intuition is that when an attacker observes the communication parameters of a victim, the random uncertainty of the victim's local dataset will decrease, which implies the attacker can extract information from the victim's local dataset to achieve more precise DRA.

\noindent {\bfseries Channel Capacity.} The key factor in describing a communication channel is channel capacity. In information theory, traditional channel capacity is defined by the maximum MI between the sending variable $\mX$ and the receiving variable $\mY$, i.e., $ C = \max_{p(\vx)} \; I(\mX; \mY)$, which is the maximum information that we can send by information coding.

In this work, the sending variable is $\mD$, while the receiving variable is $\mW_o$, which is decided by the variables $\mW_i$ and $\mD$ according to Eq.~(\ref{eq:local_sgd_FL}), hence the channel capacity can be formalized as $C = \max_{p(\vw_o)} \; I(\mD;\mW_o | \mW_i)$,
which is the upper bound of the transmitted information.

\noindent{\bfseries The reasons for choosing MI to measure the information leakage.} \qi{Prior research like QIF and g-leakage utilizes min-entropy \cite{DBLP:series/isc/AlvimCMMPS20, DBLP:conf/fossacs/Smith09} to measure the information leakage, which can only work in a white-box and time-invariant system \cite{DBLP:conf/ccs/0002CPP20}. However, FL is a black-box (e.g., deep neural networks) and time-variant (e.g., parameter updating in each round) system. Hence, these techniques are not applicable. Moreover, min-entropy is unsuitable for modeling the attack against FL, e.g., DRA. The min-entropy focuses on measuring the uncertainty of guessing the most likely output of random sources \cite{hagerty2012entropy, DBLP:journals/tifs/KimGK21}, which is more related to cryptographic systems. However, in DRA, the attacker's target is to reconstruct the whole data distribution based on the victim's sharing parameters, instead of guessing a most likely data point. Therefore, the precision of DRA attacks depends on the difference between two distributions, which must take all data of the distributions into consideration. In this scenario, the Shannon entropy, i.e., MI, which is based on the expectation metric, can accurately measure the correlation between the whole distributions, thus is more suitable for analyzing the information leakage issue under DRA. Therefore, we choose MI to measure the information leakage in
FL. }

\section{Key Observation and Method Overview}
\subsection{Key Observation}
Regarding FL, different clients jointly optimize the model by passing parameters to the server instead of raw data. Since the parameters are the mapping of the original data, the attacker can still reconstruct the private data from the parameters, thereby stealing privacy. Correspondingly, the client can preserve privacy by perturbing the transmitted parameters. Therefore, we build a channel model to calculate the privacy data contained in parameters. Moreover, based on the quantitative results, the privacy data in parameters can be flexibly adjusted to meet various privacy requirements.


In this section, we establish a formal correlation between the transmitted information, i.e., the MI, and the reconstruction error of DRA by the following theorem.

\begin{theorem}[Lower bound for reconstruction error]\label{mse_lower_bound}
    For any random variable $\mD$, $\mD \in \mathbb{R}^{d}$ and $\mW$, $\mW \in \mathbb{R}^{m}$, we have
    \begin{equation}\label{eq:mse_lower_bound}
        \mathbb{E} [ \|\mD - \hat{\mD}(\mW)\|^2 / d ] \geq \frac{e^{2h(\mD)/d}}{2\pi e} e^{-2I(\mD;\; \mW)/d},
    \end{equation}
    where $\hat{\mD}(\mW)$ is an estimator of $\mD$ constructed by $\mW$.
\end{theorem}

Specifically, Eq.~(\ref{eq:mse_lower_bound}) denotes the lower bound of MSE in DRA, which represents the optimal error for data reconstruction. In the FL scenario, we denote $\mD$ as the target data distribution and $\mW$ as the shared parameter. Therefore $h(\mD)$ is the entropy of the target data distribution, which is a constant during machine learning.
Moreover, the lower bound is negatively correlated to MI, i.e., $I(\mD; \mW)$, which indicates that if $\mW$ contains more information of target data $\mD$, i.e., a larger $I(\mD; \mW)$, the attacker can achieve a more precise reconstruction of $\mD$. Consequently, large $I(\mD; \mW)$ exacerbates the privacy issue.

In FL, MI increases as the number of optimization rounds increases, thereby enhancing the precision of DRA. Therefore, constraining the overall MI in FL is the way to restrict the precision of DRA. To this end, we build a channel model to measure the increase of MI in Sec.~\ref{sec:quanti_model}, and propose three implementation methods to limit MI within a certain threshold. 


\subsection{Method Overview}

\begin{figure}[tb]
\begin{center}
\begin{minipage}[b]{\columnwidth}
\includegraphics[width=\columnwidth]{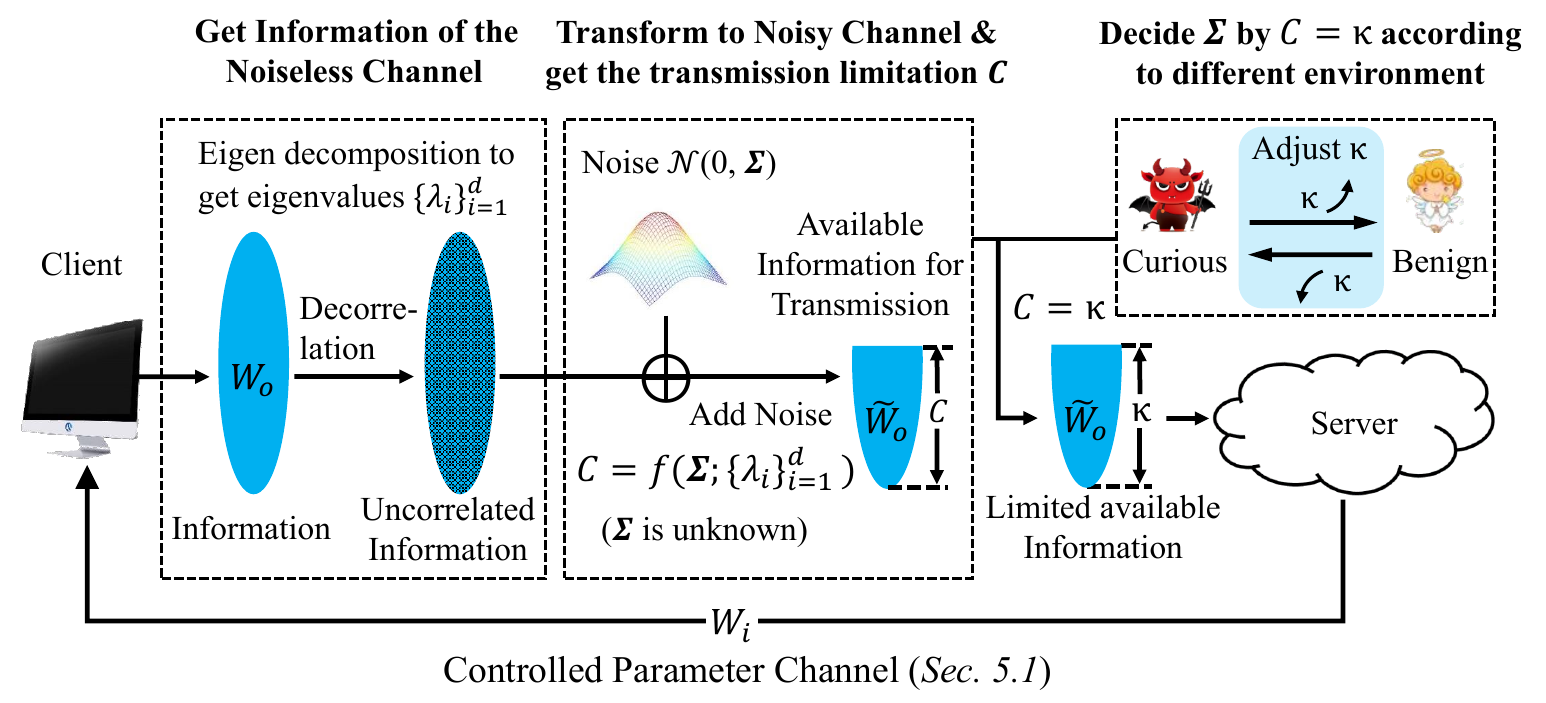}
\caption{To enhance the capability for defending against DRA in FL, we develop techniques to constrain the amount of transmitted information below a certain threshold $\kappa$.}\label{fig:motivation}
\vskip -0.2in
\end{minipage}
\end{center}
\vskip -0.2in
\end{figure}

In this study, we develop techniques to enhance FL's capability of privacy protection, in particular, defending against DRA. Normally, FL transmits information from the local dataset to the server by model parameters, i.e., the parameter channel. The transmitted information can be employed by an attacker to conduct various attacks. Hence, the objective of privacy protection is to constrain the amount of transmitted information. Specifically, for DRA, the reconstruction error is lower bounded by a function of MI (Thm.~\ref{mse_lower_bound}), which indicates that a smaller MI leads to a larger reconstruction error. Therefore, our technique is devoted to limiting the reconstruction ability by constraining the total MI in FL.

\noindent{\bfseries Threat Model.} \qi{In this work, we focus on privacy leaks incurred by DRA in FL. Specifically, the attacker aims to reconstruct the data distribution $\mD$ of a specific client (i.e., the victim) via parameters shared by the victim. We assume that the attacker can get the transmitted parameters by the victim (i.e., $\mW_i$ and $\mW_o$). 
Then the attacker can reconstruct a data distribution $\hat{\mD}(\mW_i, \mW_o)$ to approximate the target data distribution $\mD$ (Appendix~\ref{sec:detail_dra} explains the details of DRA).} 

\qi{Our goal is to constrain the private data that the attacker can obtain according to the \emph{transmitted parameters} in FL. Therefore, the attacker's ability to construct the DRA attack is limited.}

\noindent{\bfseries Controlled parameter channel.} As illustrated in Fig.~\ref{fig:motivation}, to restrict the transmitted information, we transform the noiseless parameter $\mW_o$ to a noisy Gaussian channel by adding Gaussian noise $\mathN(\bm{0}, \mSigma)$ to it. \qi{Based on the theorem of Gaussian channel in information theory~\cite{DBLP:books/daglib/0016881}}, the noisy Gaussian channel has limited capability for information transmission, which is the channel capacity. Thus, with the eigenvalues of $\mW_o$, we derive a formula $f(\mSigma)$ to characterize the channel capacity of the Gaussian channel by the maximum entropy distribution. Finally, we solve the equation $f(\mSigma)=\kappa$ to decide $\mSigma$ for constraining the transmitted information within a threshold $\kappa$. We will explain how the controlled parameter channel constrains transmitted information in Sec.~\ref{sec:quanti_model} and Sec.~\ref{subsec:ctr_parameter_space}.

\noindent{\bfseries Constraining channel capacity in the data space.} \qi{To overcome the efficiency issue caused by the high dimensional and time-variant model in FL, we theoretically transform operations of constraining channel capacity from the parameter space to the data space.} Specifically, the information contained in the resulting parameter is decided by the input data, thus constraining channel capacity can be achieved by restricting the information contained in the input data. Moreover, this transformation significantly improves the training efficiency and the model accuracy under constrained information leakage. We will explain the transformation in Sec.~\ref{subsec:ctr_data_space}.

Theoretically, when an attacker observes $\mW_i$ and $\mW_o$, the random uncertainty of the local data $\mD$ decreases, which means the attacker gets more information to conduct DRA attacks, leading to more precise reconstruction. The amount of random uncertainty reduction $\Delta I$ (i.e., information leakage) in a round can be formalized as
$$\Delta I = I(\mD; \mW_i, \mW_o) - I(\mD; \mW_i) =  I(\mD; \mW_o | \mW_i), $$
which means the attacker gets $\Delta I$ information from $\mD$ (as illustrated in Fig.~\ref{fig:unfoldinfoflow}). Moreover, this privacy leak occurs in each optimization round, which results in an increase in MI and consistently increases the risk of privacy.

As aforementioned, for defending against DRA, our technique constructs a controlled parameter channel by limiting $\Delta I$ to less than a threshold $\kappa$ for all optimization rounds. Then together with the bounded optimization rounds $n$, we provide $n \cdot \kappa$ guarantee for the total information leakage, thereby constraining the attack precision of DRA.

\section{Channel Model of the Information Leakage}\label{sec:quanti_model}
In this section, we formalize the problem of FL into a communication process based on information theory and then unfold the recurrent communication process into a time-dependent Markov Chain. Finally, we build a channel model to calculate the MI according to the unfolded communication process.

\subsection{Accumulation of Mutual Information}
As illustrated in Fig.~\ref{fig:unfoldinfoflow}, in the FL scenario, a specific client, i.e., the victim, communicates with the server through a logic channel: the parameters $\mW_{i}$ and $\mW_{o}$. Specifically, There are three different information flows: \emph{the ingress flow}, i.e., the received parameter $\mW_i$, which determines the background knowledge possessed by the server (i.e., the attacker); \emph{the egress flow}, which is the information contained in the sharing parameter $\mW_o$; and \emph{the internal flow}, i.e., local optimization process, which loads the information contained in the local dataset to the egress flow. Particularly, due to privacy requests, the victim only communicates with the server by $\mW_i$ and $\mW_o$.

The information leakage of the communication channel depends on the MI increment when the server observes $\mW_{o}$. Thus, it can be formalized as
\begin{equation}\label{eq:informationflow}
  I(\mD; \mW_i, \mW_o) = \underbrace{I(\mD;  \mW_i)}_{\text{Prior}} + \underbrace{I(\mD; \mW_o | \mW_i)}_{\text{Information Leakage ($\Delta I$)}}.
\end{equation}
Eq.~(\ref{eq:informationflow}) is an immediate result according to the chain rule of MI. It indicates that the MI between $\mD$ and the joint distribution $(\mW_i, \mW_o)$ can be divided into two parts: the prior knowledge and the information leakage $\Delta I$. \qi{In the rest of this paper, we will utilize $I(\mD; \mW_o | \mW_i)$ instead of $\Delta I$ for more comprehensible analysis.}

However, $\mW_i$ and $\mW_o$ are joint distributions of different rounds, which contain multiple local learning processes. Hence, connecting Eq. (\ref{eq:informationflow}) to the internal flow is difficult. To resolve this issue, we unfold the recurrent process to a time-dependent Markov chain. As illustrated in Fig.~\ref{fig:unfoldinfoflow}, there is only one local optimization process within a round from $\mW^{(t)}_{i}$ to $\mW^{(t+1)}_{i}$, hence we can analyze MI increment at round $t$.

Specifically, according to Eq.~(\ref{eq:local_sgd_FL}) and Eq.~(\ref{eq:aggregation_FL}), the relationship between $\mW^{(t)}_{i}$ and $\mW^{(t)}_{o}$ in Fig.~\ref{fig:unfoldinfoflow} is
\begin{equation}\label{eq:local_sgd}
  \mW^{(t)}_{o} = \mW^{(t)}_{i} - \eta \cdot \nabla_{\mW}F(\mW^{(t)}_{i};\mD),
\end{equation}
which represents the local learning process with SGD. While for $\mW^{(t)}_{o}$ and $\mW^{(t+1)}_{i}$, the relationship is
\begin{equation}\label{eq:aggregation}
  \mW^{(t+1)}_{i} = \mathA^{(t)}(\mW^{(t)}_{o}; \mV^{(t)}),
\end{equation}
where $\mV^{(t)}$ are the variables uploaded by clients other than the victim.

Then based on the unfolded process, we transform the MI in Eq. (\ref{eq:informationflow}) to the joint mutual information as
\begin{align}\label{eq:unfolded_mi}
&&&I(\mD; \; \mW^{(0)}_{i}, \mW^{(0)}_{o}, \cdots , \mW^{(n)}_{i}) \notag\\
&=&& \textstyle\sum_{t=0}^{n} I(\mD; \mW^{(t)}_{i}| \mW^{(t-1)}_{o},\mW^{(t-1)}_{i},\cdots,\mW^{(0)}_{o},\mW^{(0)}_{i}) \notag\\
&&&+ \textstyle\sum_{t=0}^{n-1} I(\mD; \mW^{(t)}_{o}| \mW^{(t)}_{i},\mW^{(t-1)}_{o},\cdots,\mW^{(0)}_{o},\mW^{(0)}_{i}) \notag\\
&=&& \underbrace{\textstyle\sum_{t=0}^{n-1} I(\mD; \mW^{(t)}_{o}| \mW^{(t)}_{i})}_{\Gamma_{client}} + \underbrace{\textstyle\sum_{t=0}^{n-1} I(\mD; \mW^{(t+1)}_{i}| \mW^{(t)}_{o})}_{\Gamma_{server}} \notag\\
&&&+ \underbrace{I(\mD; \mW^{(0)}_{i})}_{\text{Prior}},
\end{align}
where the first equality depends on the chain rule of MI and the last equality is an immediate consequence of the Markov property of the learning process.

\begin{figure}[tb]
\begin{center}
\begin{minipage}[b]{\columnwidth}
\centering
\includegraphics[width=\columnwidth]{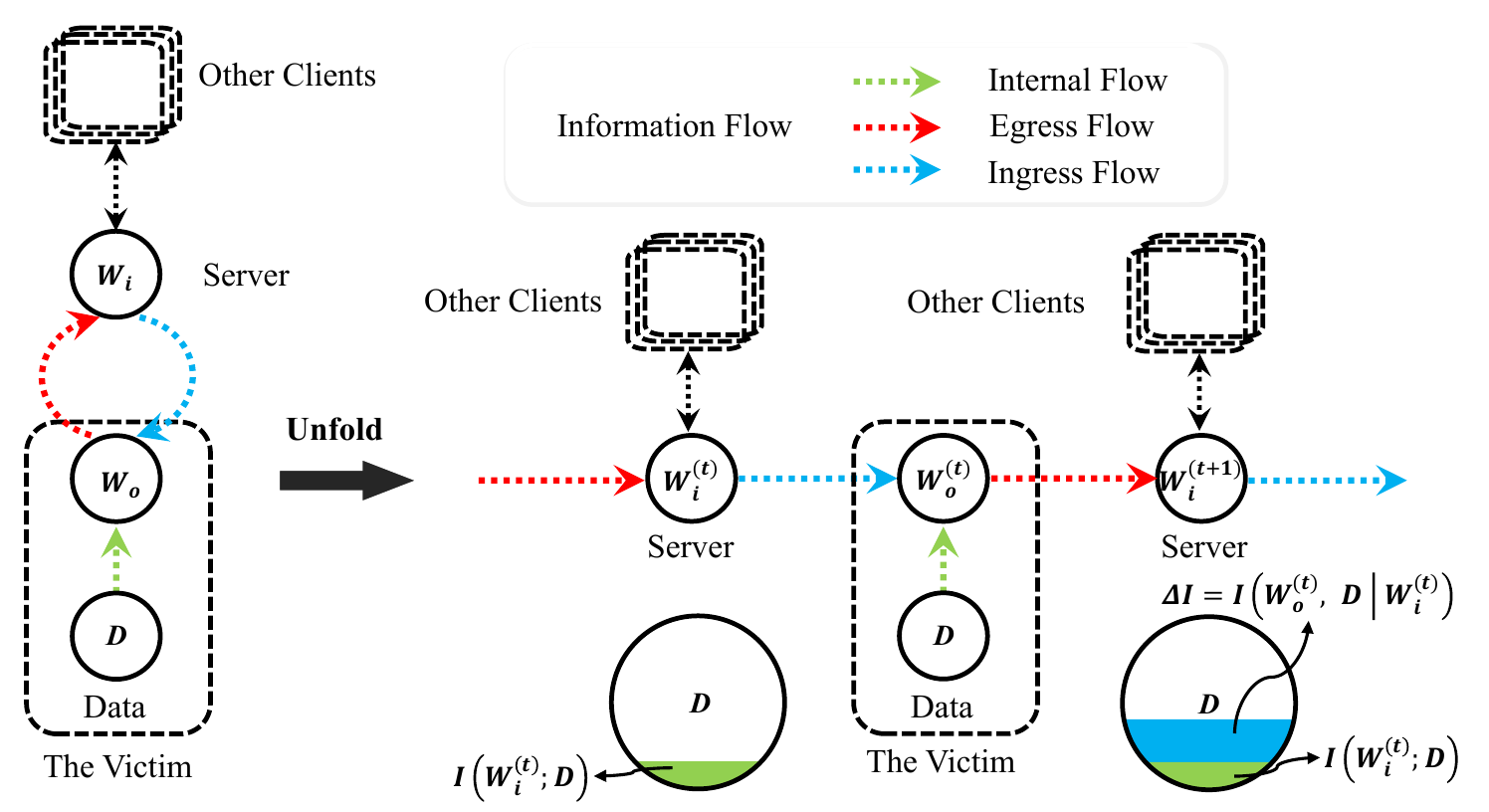}
\vskip -0.1in
\caption{The process of FL can be unfolded to a time-dependent Markov chain. Hence we can analyze the mutual information in a round from $\mW^{(t)}_{i}$ to $\mW^{(t+1)}_{i}$.}\label{fig:unfoldinfoflow}
\vskip -0.2in
\end{minipage}
\end{center}
\vskip -0.2in
\end{figure}

Specifically, Eq.~(\ref{eq:unfolded_mi}) indicates that the overall MI is comprised of three parts: the prior, $\Gamma_{server}$, and $\Gamma_{client}$.

\subsection{Analysis of Mutual Information}
Regarding Eq.~(\ref{eq:unfolded_mi}), the information can be divided into three parts: the prior, $\Gamma_{server}$, and $\Gamma_{client}$. The prior knowledge $I(\mD; \mW^{(0)}_{i})$ is decided by the background knowledge of the attacker before the learning process.

For $\Gamma_{server}$, it represents the aggregation process on the server, \qi{hence it cannot be controlled by the local learning process. On the contrary, it can be exploited by an attacker}. Without loss of generality, we focus on the general term $I(\mD; \mW^{(t+1)}_{i}| \mW^{(t)}_{o})$, $t \in \{0,\cdots , n-1\}$, which is the information increment on the server. Based on the relationship between the MI and the entropy, we rewrite it as
\begin{align*}
    &\; I(\mD; \mW^{(t+1)}_{i}| \mW^{(t)}_{o}) \\
  = &\; h(\mW^{(t+1)}_{i}| \mW^{(t)}_{o}) - h(\mW^{(t+1)}_{i}| \mD, \mW^{(t)}_{o}) \\
  = &\; h(\mathA^{(t)}(\mW^{(t)}_{o}; \mV^{(t)})| \mW^{(t)}_{o}) - h(\mathA^{(t)}(\mW^{(t)}_{o}; \mV^{(t)})| \mD, \mW^{(t)}_{o}),
\end{align*}
the last equality is a substitution according to Eq.~(\ref{eq:aggregation}). The formula indicates that $I(\mD; \mW^{(t+1)}_{i}| \mW^{(t)}_{o})$ is decided by $\mV^{(t)}$, which are variables independent of local learning process.

Specifically, if $\mV^{(t)}$ is independent of $\mD$, then $\mD\rightarrow\mW^{(t)}_{o}\rightarrow\mW^{(t+1)}_{i}$ forms a Markov chain. Therefore, given the observation of $\mW^{(t)}_{o}$, $\mD$ is conditional independent of $\mW^{(t+1)}_{i}$, i.e., $I(\mD; \mW^{(t+1)}_{i}| \mW^{(t)}_{o})=0$. In this situation, $\Gamma_{server}=0$, which means the server (i.e., the attacker) cannot affect DRA attacks. \qi{On the contrary, if $\mV^{(t)}$ has the information of $\mD$ that is independent of $\mW^{(t)}_{o}$, i.e., the attacker can get auxiliary information other than the local learning process, the mutual information $I(\mD; \mW^{(t+1)}_{i}| \mW^{(t)}_{o})>0$, which increases the risk of information leakage for the rest communication rounds when $T\geq t+1$. In this situation, the server (i.e., the attacker) can increase the risk of privacy by utilizing the information collected by means other than FL.}

\qi{However, the analysis of $\Gamma_{server}$ is beyond the protection of the local learning process in FL, since $\Gamma_{server}$ is related to the auxiliary information for the attacker to conduct the DRA attack.
Particularly, our target is to bound $\Gamma_{client}$, which is the information leakage of the local learning process in FL. Sec.~\ref{sec:info_ctr} indicates that regardless of $\Gamma_{server}>0$ or not, $\Gamma_{client}$ is constrained by our methods.}

\qi{Finally, we put emphasis on $\Gamma_{client}$, the most important part that is correlated to information leakage of the local learning process in FL.} Similarly, we focus on the general term of $\Gamma_{client}$, i.e., $I(\mD; \mW^{(t)}_{o}| \mW^{(t)}_{i})$, which is the MI increment at round $t$, then we have
\begin{equation}\label{eq:gamma_in}
  I(\mD; \mW^{(t)}_{o}| \mW^{(t)}_{i}) = h(\mW^{(t)}_{o}| \mW^{(t)}_{i}) - h(\mW^{(t)}_{o}| \mD, \mW^{(t)}_{i}).
\end{equation}
Based on Eq~(\ref{eq:local_sgd}), if $\mW^{(t)}_{i}$ is observed, $\mW^{(t)}_{o}$ is decided by a deterministic function of $\mD$, which has a finite entropy. Moreover, if $\mW^{(t)}_{i}$ and $\mD$ are both observed, $\mW^{(t)}_{o}$ is deterministic, thereby $h(\mW^{(t)}_{o}| \mD, \mW^{(t)}_{i}) \to -\infty$. The result is reasonable since the volume of a constant's support set\footnote{The volume of support set for a random variable $\mX$ is $2^{h(\mX)}$ \cite{DBLP:books/daglib/0016881}.} goes to $0$, i.e., $2^{-\infty} = 0$. In this case, $I(\mD; \mW^{(t)}_{o}| \mW^{(t)}_{i}) \to +\infty$, which means a noiseless channel results in unlimited risk of privacy leaks.

To limit the information leakage, we add a Gaussian noise to $\mW^{(t)}_{o}$, which transforms Eq.~(\ref{eq:gamma_in}) to a noisy egress flow. Specifically, we turn to analyze
\begin{equation}\label{eq:noise_channel}
  \widetilde{\mW}^{(t)}_{o} = \mW^{(t)}_{o} + \bm{\xi},
\end{equation}
where $\bm{\xi}\sim \mathN(\bm{0}, \mSigma)$, and $\mW^{(t)}_{o}=\lim_{\mSigma \to \bm{0}}\widetilde{\mW}^{(t)}_{o}$. Then the property of $\widetilde{\mW}^{(t)}_{o}$ is implied by following lemma.

\begin{lemma}[Maximum entropy distribution]\label{lemma:upper_bound_entropy}
  Let $\mX$ be a continuous random vector with $\mathE[\mX] = \vmu_{\mX}$, $Cov(\mX) = \mSigma_{\mX}$. Let $\mY \sim \mathN(\vmu_{\mY}, \mSigma_{\mY})$ be a Gaussian random variable that is independent with $\mX$, then $h(\mX+\mY)$ achieves its maximum when $\mX \sim \mathN(\vmu_{\mX}, \mSigma_{\mX})$.
\end{lemma}
%
%

To analyze the noisy egress flow after transformation, we rewrite Eq.~(\ref{eq:gamma_in}) as
\begin{equation}\label{eq:extened_gamma_in}
    I(\mD; \widetilde{\mW}^{(t)}_{o}| \mW^{(t)}_{i}) = h(\mW^{(t)}_{o}+\bm{\xi}| \mW^{(t)}_{i}) - h(\mW^{(t)}_{o}+\bm{\xi}| \mD, \mW^{(t)}_{i}).
\end{equation}

As $\bm{\xi}$ is a Gaussian variable, Lemma~\ref{lemma:upper_bound_entropy} implies that the first term on the right-hand side of Eq.~(\ref{eq:extened_gamma_in}), i.e., $h(\mW^{(t)}_{o}+\bm{\xi}| \mW^{(t)}_{i})$, is upper bounded by the entropy of the Gaussian distribution. Moreover, when $\mD$ and $\mW_i^{(t)}$ are both observed, $\mW^{(t)}_o$ is a constant, which means the only randomness of $\widetilde{\mW}^{(t)}_{o}=\mW^{(t)}_{o}+\bm{\xi}$ comes from $\bm{\xi}$, thus the second term on the right-hand side of Eq.~(\ref{eq:extened_gamma_in}) is $h(\mathN(\bm{0}, \mSigma))$. Let $\mathE[\mW^{(t)}_{o}]=\vmu_{\mW}$ and $Cov(\mW^{(t)}_{o})=\mSigma_{\mW}$, we have the upper bound of Eq.~(\ref{eq:extened_gamma_in}) as
\begin{equation}\label{eq:upper_bound_noise_channel}
  I(\mD; \widetilde{\mW}^{(t)}_{o}| \mW^{(t)}_{i} ) \leq  h(\mathN(\vmu_{\mW}, \mSigma_{\mW}+\mSigma))-h(\mathN(\bm{0}, \mSigma)).
\end{equation}

In practice, the distribution of $\mW^{(t)}_{o}$ is extremely complex and time-variant, so we utilize upper bound (\ref{eq:upper_bound_noise_channel}) to limit the information leakage. The important parts of Eq.~(\ref{eq:upper_bound_noise_channel}) are the covariance matrixes of $\mW^{(t)}_{o}$ and $\bm{\xi}$. Moreover, Sec.~\ref{sec:info_ctr} indicates that regardless of $\mSigma_{\mW}$, we can restrict $I(\mD;\widetilde{\mW}^{(t)}_{o} | \mW^{(t)}_{i} )\leq \kappa$, $\forall\, \kappa>0$, by deciding $\mSigma$.


Finally, for multiple local updates, the scenario is slightly different. We denote the number of local steps in one communication round as $E$, and the total number of local steps for all communication rounds as $n$, where $n$ is divisible into $E$. These notations imply that the number of communications is $T = \frac{n}{E}$. Then the issue can be transformed to the one-step case by a time-dependent aggregation method as
\begin{align}\label{time_dependent_agg}
    \mW^{(t)}_{i} = \begin{cases}
            \mathA^{(t)}(\mW^{(t-1)}_{o};\; \mV^{(t-1)}),  & \mbox{if } E\,|\,t\\
            \mW^{(t-1)}_{o}, & \mbox{otherwise}.
          \end{cases}
\end{align}

Where $E\,|\,t$ represents we make aggregation on the server every $E$ steps. If we set $E=n$, the FL problem reduces to classical ML without collaboration.

\qi{Based on the former analysis, the new aggregation rule only changes $\Gamma_{server}$ in Eq.~(\ref{eq:unfolded_mi}), while $\Gamma_{client}$ remains the same. Hence, the privacy analysis for the local optimization process is identical to classical ML. Therefore, our analysis of $\Gamma_{client}$ will focus on classical ML in the remainder of this paper.}

\section{Controlled Parameter Channel}\label{sec:info_ctr}

Based on the former analysis, the important part of defending against DRA is $I(\mD; \widetilde{\mW}^{(t)}_{o}| \mW^{(t)}_{i})$, which represents the MI increment at round $t$. The key parameters to restrict the MI increment are $\mSigma_{\mW}$ and $\mSigma$. In this section, we first propose a method for deciding $\mSigma$ that ensures $I(\mD; \widetilde{\mW}^{(t)}_{o}| \mW^{(t)}_{i})\leq \kappa$.
Then, we transform the operations for constraining MI from the parameter space to the data space and propose three implementation methods for constraining the channel capacity.
Finally, we analyze existing techniques with our theoretical results in defending against DRA.

\subsection{Controlled Channel Capacity}\label{subsec:ctr_parameter_space}

Channel capacity is the maximum ability to transmit information within a single round, i.e., $\max \;I(\mD;\widetilde{\mW}^{(t)}_{o} | \mW^{(t)}_{i} )$, which is the key parameter for constraining the information leakage. 


To constrain the channel capacity, we have Thm~\ref{thm:channel_capacity}.
\begin{theorem}[Channel capacity]\label{thm:channel_capacity}
  Let $\widetilde{\mW}^{(t)}_{o} = \mW^{(t)}_{o}+\sqrt{\sigma} \cdot \bm{\xi}$, where $\sigma \geq 0$ and $\bm{\xi}\sim \mathN(\bm{0}, \rmI)$, if $\vmu^{(t)}$ and $\mSigma^{(t)}$ are the mean vector and covariance matrix of $\mW^{(t)}_{o}$ when $\mW^{(t)}_{i}$ is observed, we have $I(\mD; \widetilde{\mW}^{(t)}_{o} | \mW^{(t)}_{i} ) \leq f^{(t)}(\sigma)$, where
  \begin{equation}\label{eq:info_func}
    f^{(t)}(\sigma) := \frac{1}{2}\sum_{i=1}^{d}\ln\frac{\lambda^{(t)}_i + \sigma}{\sigma}, \; \sigma \in (0, \; +\infty).
  \end{equation}
  where $\lambda^{(t)}_i$ is the i-th eigenvalue of the covariance matrix $\mSigma^{(t)}$ and $d$ represents the dimension of $\mW^{(t)}_{o}$.
\end{theorem}
Based on Lemma~\ref{lemma:upper_bound_entropy}, $I(\mD, \; \widetilde{\mW}^{(t)}_{o} | \mW^{(t)}_{i} )$ achieves its maximum $f^{(t)}(\sigma)$ when $\mW^{(t)}_{o}$ conforms to the Gaussian distribution. According to the Central Limit Theorem, if we use mini-batch SGD, the distribution of $\mW^{(t)}_{o}$ converges to the Gaussian distribution, which means upper bound~(\ref{eq:info_func}) becomes tighter when we utilize larger batch size for local training.


Regarding $\Gamma_{client}$, if we denote $C^{(t)} = f^{(t)}(\sigma^{(t)})$, where $\sigma^{(t)}$ represents $\sigma$ at time $t$, we have $I(\mD; \widetilde{\mW}^{(t)}_{o}| \mW^{(t)}_{i}) \leq C^{(t)}$, which indicates 
$\mW^{(t)}_{i}$ and $\widetilde{\mW}^{(t)}_{o}$ form a communication channel with channel capacity $C^{(t)}$.

Specifically, according to Eq.~(\ref{eq:info_func}), $C^{(t)}$ is decided by two components: $\{\lambda^{(t)}_i\}^{d}_{i=1}$ and $\sigma^{(t)}$. First, $\{\lambda^{(t)}_i\}^{d}_{i=1}$ represent the eigenvalues of $\mSigma^{(t)}$. Based on Eq.~(\ref{eq:local_sgd}), $\mW^{(t)}_{o}$ is related to 
$\mW^{(t)}_{i}$, which is the parameter received from the server. Therefore, the server (i.e., the attacker) can craft $\mW^{(t)}_{i}$ to get more information from the local dataset (as displayed in Fig.~\ref{fig:channel_capacity}).

\begin{figure*}[tb]
\begin{center}
\begin{minipage}[b]{0.29\textwidth}
\vskip -0.1in
\includegraphics[width=\textwidth]{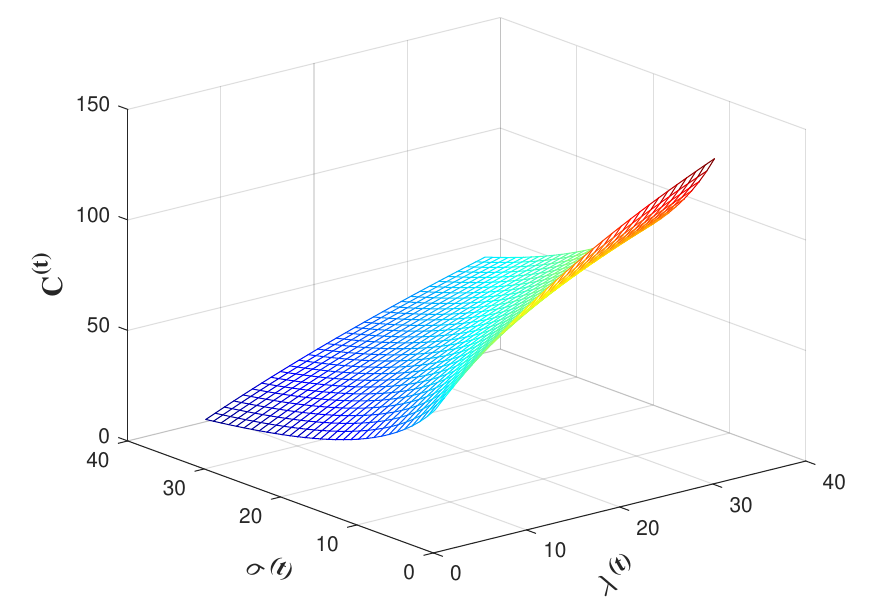}
\vskip -0.1in
\caption{The channel capacity $C^{(t)}$ is the maximum MI increment at round $t$. It is an increasing function of $\lambda^{(t)}$ and a decreasing function of $\sigma^{(t)}$.}\label{fig:channel_capacity}
\end{minipage}
\hfill
\begin{minipage}[b]{0.29\textwidth}
\vskip -0.1in
\includegraphics[width=\textwidth]{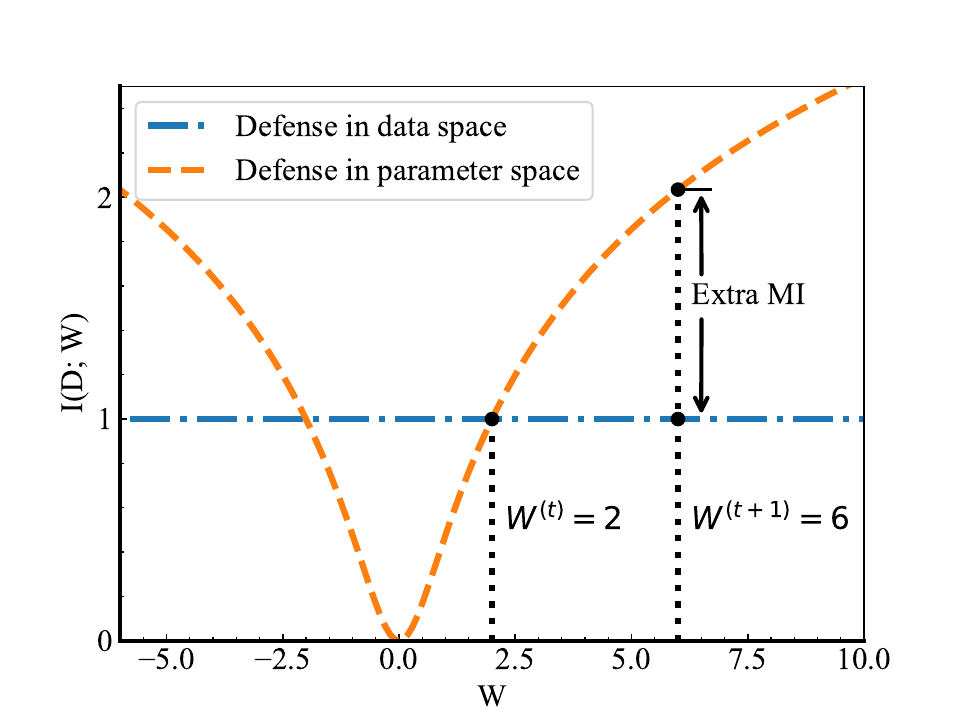}
\vskip -0.1in
\caption{\qi{A toy example to explain the rationale for constraining in data space, which is equivalent to adding an adaptive noise to the parameter.}}\label{fig:exp_rationale_}
\end{minipage}
\hfill
\begin{minipage}[b]{0.34\textwidth}
\vskip -0.1in
\centering
\includegraphics[width=\textwidth]{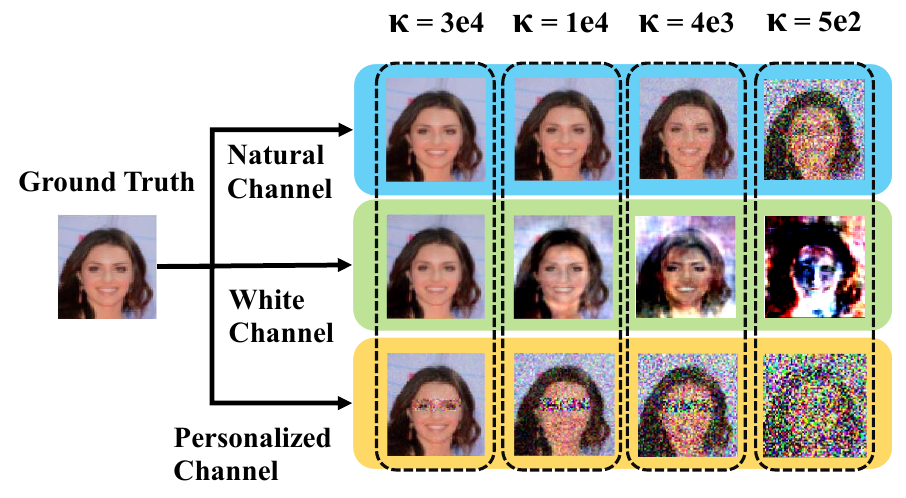}
\caption{Visualizations for CelebA when we apply different channel implementations (Natural, White, and Personalized) and utilize different channel capacities.}\label{fig:Overview_Channel_Capacity}
\end{minipage}
\end{center}
\vskip -0.2in
\end{figure*}

Second, $\sigma^{(t)}$ is related to the added noise. Based on the monotonicity of $C^{(t)}=f^{(t)}(\sigma^{(t)})$, we conclude that for any $\mW^{(t)}_{i}$, there exists an unique $\sigma^{(t)}$ that satisfies $f^{(t)}(\sigma^{(t)}) = \kappa$, where $\kappa \geq 0$ is a certain threshold. 

If we denote $\lambda^{(t)}_{i}=\lambda^{(t)}$, $i\in\{1, \cdots,d\}$, we have $C^{(t)}=\frac{d}{2}\ln(\frac{\lambda^{(t)}+\sigma^{(t)}}{\sigma^{(t)}})$. In this scenario, the channel capacity can be displayed as Fig.~\ref{fig:channel_capacity}. \qi{Specifically, the channel capacity is an increasing function of $\lambda^{(t)}$ and a decreasing function of $\sigma^{(t)}$. Moreover, it is worth noting that even the server can change the channel capacity by crafting $\mW^{(t)}_{i}$, the victim can constrain the transmitted information within $\kappa$ by solving $f^{(t)}(\sigma^{(t)}) = \kappa$ after receiving $\mW^{(t)}_{i}$, which decides the added noise $\sigma^{(t)}$.}

\subsection{Limiting Channel Capacity in Data Space}\label{subsec:ctr_data_space}

For a controlled parameter channel, our target is to constrain the transmitted information at round $t$, i.e., $I(\mD; \widetilde{\mW}^{(t)}_{o}| \mW^{(t)}_{i})$. Based on Thm.~\ref{thm:channel_capacity}, there are two steps for solving the equation $C^{(t)} = \kappa$ in the parameter space: the eigen-decomposition of $\mSigma^{(t)}$ and solving the high order equation with corresponding eigenvalues. Whereas, $\mW^{(t)}_{o}$ is a high-dimensional and time-variant parameter, which leads to an extremely large number of calculations. Hence, the implementation of the aforementioned method in the parameter space is computationally expensive.

Regarding $\mW^{(t)}_{o}$ in Fig.~\ref{fig:unfoldinfoflow}, when $\mW^{(t)}_{i}$ is observed, it is conditional independent with the previous parameters $\mW^{(s)}_{*}$, $\forall s<t$, and all of the previous local learning processes. Therefore, if $\mW^{(t)}_{i}$ is observed, $\mW^{(t)}_{o}$ is purely decided by the local learning process at round $t$, i.e., $\mD \rightarrow \mW^{(t)}_{o} | \mW^{(t)}_{i}$. Then based on these properties, we design a random function $M(\cdot)$ to map the raw data $\mD$ to the noisy data $\widetilde{\mD}$, and then utilize the noisy data for the local training, i.e., $\widetilde{\mW}^{(t)}_{o}=\mW^{(t)}_{i} - \eta \cdot \nabla_{\mW}F(\mW^{(t)}_{i};\widetilde{\mD})$.

If we use the random function $M(\cdot)$ before the local learning process, the variables form a Markov Chain, i.e., $\mD \rightarrow \widetilde{\mD} \rightarrow \mW^{(t)}_{o} | \mW^{(t)}_{i}$. According to the DPI \cite{DBLP:books/daglib/0016881}, we have
\begin{equation}\label{eq:data_para_transfer}
   I(\mD; \mW^{(t)}_{o}| \mW^{(t)}_{i}) \leq I(\mD; \widetilde{\mD}| \mW^{(t)}_{i}) = I(\mD; \widetilde{\mD}),
\end{equation}
where the last equality results from the independence between $M(\cdot)$ and $\mW^{(t)}_{i}$. Therefore, we can bound $I(\mD; \widetilde{\mD})$ so as to restrict $I(\mD; \mW^{(t)}_{o}| \mW^{(t)}_{i})$. Moreover, bounding $I(\mD; \widetilde{\mD})$ enables us to limit the channel capacity in the data space. Specifically, we use  $\widetilde{\mD} = M(\mD)=\mD + \vxi$ as the random function, where $\vxi \sim \mathN(\bm{0}, \mSigma_{\vxi})$, and the key parameter is $\mSigma_{\vxi}$.

\noindent{\bfseries \qi{The rationale of constraining transmitted information in data space.}} \qi{In addition to the upper bound derived by DPI, we can explain the rationale of constraining in data space by Taylor's expansion for a more comprehensible analysis. Specifically, if we use a Gaussian noise in the data space to constrain the information leakage, we can expand the gradient mapping as follows}
\begin{align}\label{eq:rationale_data_space}
    &\nabla_{\mW} F(\mW;\mD + \bm{\xi}) &=& \nabla_{\mW} F(\mW;\mD)+ \nabla_{\mD} \nabla_{\mW} F(\mW;\mD)^{\text{T}} \cdot \bm{\xi} \notag\\
                                       &&&+ O(\|\bm{\xi}\|^2).
\end{align}

\qi{Eq.~(\ref{eq:rationale_data_space}) indicates that constraining the transmitted information in the data space is equivalent to adding an adaptive noise $\nabla_{\mD} \nabla_{\mW} F(\mW;\mD)^{\text{T}} \cdot \bm{\xi}$ to the parameter.}

\qi{Moreover, the coefficient $\nabla_{\mD} \nabla_{\mW} F(\mW;\mD) $
is the variation of $\nabla_{\mW} F(\mW;\mD)$. If $\nabla_{\mW} F(\mW;\mD)$ changes significantly according to the data $\mD$, which means the gradient has a high distinction degree with regard to the data, i.e., the gradient leaks more information, the large coefficient will provide strong privacy protection. Otherwise, the small coefficient leads to better utility without violating the privacy requirements.}

\qi{Additionally, Fig.~\ref{fig:exp_rationale_} provides a toy example for understanding the rationale of constraining in the data space. Specifically, we set $D \sim \mathcal{N}(-1, 1)$ and our target is to decide the model $W$ that minimizes $F(W; D)=W^2 D$ through gradient descent. Particularly, we require the information leakage to be less than $1$, i.e., $C^{(t)}=1$. As illustrated in Fig.~\ref{fig:exp_rationale_}, when we use the method in the parameter space to limit $C^{(t)}=1$ at $W^{(t)}=2$, i.e., adding the noise $\xi \sim \mathcal{N}(0, \frac{16}{e^{2}-1})$ to the gradient $\nabla_{W}F(W;D)=2\,WD$, the resulting noise cannot guarantee $C^{(t+1)}=1$ at next round $W^{(t+1)}=6$ (we set the learning rate to $1$), which implies that we need to recalculate the covariance matrix for the noise at each round. On the contrary, constraining in the data space (adding the noise $\xi \sim \mathcal{N}(0, \frac{1}{e^{2}-1})$ to the data $D$) results in an adaptive noise to guarantee $C^{(t)}=1$ for all $t$ regardless of $W^{(t)}$, leading to the $O(1)$ time-complexity for achieving the privacy requirement.}

By incorporating different prior knowledge, we propose three implementation methods for deciding $\mSigma_{\vxi}$.

\noindent {\bfseries Natural Channel.} \qi{For Natural Channel, the relative importance of different data attributes is naturally decided by the data itself.} Specifically, based on Thm.~\ref{thm:channel_capacity}, with substituting $\widetilde{\mD}$ for $\widetilde{\mW}^{(t)}_{o}$, we have
\begin{equation}\label{eq:data_channel_capacity}
  I(\mD; \widetilde{\mD})\leq f(\sigma)=\frac{1}{2}\sum_{i=1}^{d}\ln\frac{\lambda_i + \sigma}{\sigma}, \; \sigma \in (0, \; +\infty),
\end{equation}
where $\lambda_i$ and $d$ represents the $i$-th eigenvalue and the dimension of data $\mD$, respectively. In this case, we chose $\mSigma_{\vxi}=\sigma \rmI$.
Moreover, $f(\sigma)$ is a monotone function of $\sigma$, hence there is a unique $\sigma$ satisfying $f(\sigma)=\kappa$. Meanwhile, this $\sigma$ guarantees $I(\mD; \widetilde{\mD})\leq f(\sigma)=\kappa$. Combining it with Eq.~(\ref{eq:data_para_transfer}), we conclude that the information leakage at time $t$ is less than $\kappa$. However, $f(\sigma)=\kappa$ is a polynomial equation of order $d$, hence we need to solve it with numerical methods (e.g., binary search). In summary, the process of Natural Channel is
\begin{enumerate}
\setlength{\itemsep}{2pt}
\setlength{\parsep}{0pt}
\setlength{\parskip}{0pt}
  \item Make the eigen-decomposition of $\mSigma_{\mD}$, which is the covariance matrix of data $\mD$, to get $\Lambda = diag(\lambda_1,\cdots,\lambda_d)$
  \item Solve the equation $f(\sigma)=\kappa$ to get $\sigma$ with binary search
  \item Get $\widetilde{\mD} = \mD + \vxi$, where $\vxi \sim \mathN(\bm{0},\; \sigma \rmI)$
\end{enumerate}

\noindent {\bfseries White Channel.}
\qi{In the method of the White Channel, we treat the relative importance of all attributes to be equal, which provides much stronger protection for the local dataset.} For such a purpose, we add a constraint to Eq.~(\ref{eq:data_channel_capacity}) as
\begin{align}\label{eq:white_channel}
  &f(\mPsi) = \frac{1}{2}\sum_{i=1}^{d}\ln\frac{\lambda_i + \sigma_i}{\sigma_i} = \kappa \\
  &s.t. \quad \ln\frac{\lambda_i + \sigma_i}{\sigma_i}=\ln\frac{\lambda_j + \sigma_j}{\sigma_j},\; \text{for } 1\leq i<j\leq d ,\notag
\end{align}
where $\mPsi=\diag (\sigma_1, \cdots, \sigma_d)$ represents the eigenvalues of $\mSigma_{\vxi}$. Then we can get $\sigma_i = \frac{\lambda_i}{\exp(2\kappa/d)-1}$ by solving Eq~(\ref{eq:white_channel}). Finally, we need to transform $\mPsi$ back to $\mSigma_{\vxi}$. According to Eq.~(\ref{eq:sigma_decomp}) in the proof of Thm.~\ref{thm:channel_capacity}, we have $\mSigma_{\mD} = \mQ \bm{\Lambda} \mQ^T$, where $\mQ$ is an orthogonal matrix composed with the eigenvectors of $\mSigma_{\mD}$. Then with a similar process, we have $\mSigma_{\vxi} = \mQ \mPsi \mQ^T$. Therefore, the typical process of the White Channel is
\begin{enumerate}
\vskip -0.1in
\setlength{\itemsep}{2pt}
\setlength{\parsep}{0pt}
\setlength{\parskip}{0pt}
  \item Make the eigen-decomposition of $\mSigma_{\mD}$ to get $\Lambda = diag(\lambda_1,\cdots,\lambda_d)$ and the eigenspace $\mQ$
  \item Get $\mPsi=\diag (\sigma_1, \cdots, \sigma_d)$ by $\sigma_i = \frac{\lambda_i}{\exp(2\kappa/d)-1}$
  \item Get $\mSigma_{\vxi} = \mQ \mPsi \mQ^T$
  \item Get $\widetilde{\mD} = \mD + \vxi$, where $\vxi \sim \mathN(\bm{0},\; \mSigma_{\vxi})$
\end{enumerate}


\noindent {\bfseries Personalized Channel.} 
In practice, the relative importance of different attributes in data is different. For example, \cite{yang2022digital} claims that for online diagnosis, our target is to extract disease-relevant features but remove identity features from the facial images of patients, which means we add more noise to the identity features compared to the disease-relevant features. For analyzing the problem, we assume the relative importance of different dimensions is $\vbeta=(\beta_0, \cdots, \beta_{d-1})^T$, e.g., (height, weight)=(1, 2) represents the relative importance of the height to the weight is $1:2$.

The relative importance decides the level of noise addition. In other words, if the attribute is more important, we add more noise to it for stronger protection. To utilize the prior knowledge $\vbeta$, we chose $\mSigma_{\vxi}=\sigma \cdot \diag(\vbeta)$. According to the proof of Thm.~\ref{thm:channel_capacity}, $\sigma$ is decided by

\begin{equation}\label{eq:personalized_channel_equation}
  f(\sigma) = \frac{\ln\det[\mSigma_{\mD} + \sigma \cdot \diag(\vbeta)]}{2\sum_{i=1}^{d}\ln (\sigma \cdot \beta_{i})} = \kappa.
\end{equation}
However, the intrinsic correlations of local data bring difficulties in solving Eq.~(\ref{eq:personalized_channel_equation}), so we derive the following theorem for simplifying the calculation.

\begin{theorem}[Upper bound of Personalized Channel.]\label{thm:upper_bound_personalized}
\vskip -0.05in
   Let $\mSigma = \mSigma_{d-1}$, we can rewrite $\mSigma$ as
   \begin{align}\label{matrix_decomp_pre}
   \mSigma_{i} = &\left( \begin{array}{cc}
                              \mSigma_{i-1} &  \vrho_{i}\\
                                 \vrho_{i}^\mathrm{T} &   c_{i,i}
                            \end{array}\right),
   \; i\in \{0,\cdots,d-1\},
   \end{align}
   where $d$ represent the dimension of $\mSigma$, then we have
   \begin{align}\label{eq:upper_bound_personalized}
     &&& \ln\det(\mSigma + \diag(\vbeta))\leq \min[\sum_{i=0}^{d-1}\ln(c_{i,i}+\beta_{i}),\sum_{i=0}^{d-1}\ln(u_i+k_i)], \notag \\
     &&& u_i  =c_{i,i}-\vrho_i^\mathrm{T}\mSigma^{-1}_{i-1}\vrho_i,\; k_i =\beta_{i}+(\mSigma^{-1}_{i-1}\vrho_i)^\mathrm{T}\diag(\vbeta)(\mSigma^{-1}_{i-1}\vrho_i). \notag
   \end{align}
\vskip -0.05in
\end{theorem}

Combining Thm.~\ref{thm:upper_bound_personalized} with Eq.~(\ref{eq:personalized_channel_equation}), we can get an upper bound of the channel capacity for the Personalized Channel
$$ U(\sigma)=\frac{\min[\sum_{i=0}^{d-1}\ln(c_{i,i}+\sigma\beta_{i}), \sum_{i=0}^{d-1}\ln(u_i+\sigma k_i)]}{2\sum_{i=1}^{d}\ln (\sigma \cdot \beta_{i})}.$$

Moreover, $U(\sigma)$ decouples $\mSigma$ and $\diag(\vbeta)$,
hence we can use a pre-processing to reduce the calculation of solving $U(\sigma)=\kappa$. The process of the Personalized Channel is
\begin{enumerate}
\setlength{\itemsep}{2pt}
\setlength{\parsep}{0pt}
\setlength{\parskip}{0pt}
  \item Get $u_i$ and $k_i$ according to Thm.~(\ref{thm:upper_bound_personalized})
  \item Solve the equation $U(\sigma)=\kappa$ to get $\sigma$ with binary search
  \item Get $\mSigma_{\vxi} = \sigma \cdot \diag(\vbeta)$
  \item Get $\widetilde{\mD} = \mD + \vxi$, where $\vxi \sim \mathN(\bm{0},\; \mSigma_{\vxi})$
\end{enumerate}

Finally, we visualize $\widetilde{\mD}$ of different implementations according to different channel capacities in Fig.~\ref{fig:Overview_Channel_Capacity}. \qi{Moreover, we discuss the advantages and the disadvantages of different channel implementations in Appendix~\ref{sec:advantage_and_disadvantage}.}

\qi{Furthermore, constraining in the data space brings two advantages: first, compared to the parameter space, data space is white-box, low-dimensional, and time-invariant. Second, in the data space, the relative importance of attributes is preserved, which makes it easier to leverage prior knowledge.}

\noindent{\bfseries Guidelines for noise injection.} \qi{In summary, for defending against DRA, the important target is to restrict the transmitted information by noise addition. 
Here we provide guidelines for the noise injection:}
\begin{itemize}
\setlength{\itemsep}{2pt}
\setlength{\parsep}{0pt}
\setlength{\parskip}{0pt}
    \item \qi{For privacy-enhancing techniques (e.g., DP and gradient compression) in FL, the reconstruction error for DRA is theoretically above a threshold (Thm.~\ref{mse_lower_bound}) when we restrict the transmitted information within $\kappa$. It can be achieved by solving $f(\sigma)=\kappa$, where $f(\sigma)$ is defined in Eq.~(\ref{eq:data_channel_capacity}). This equation decides the injection noise with corresponding statistics.}
    \item \qi{For defending against DRA in FL, we can inject noise to the training data instead of transmitted parameters. This transformation can produce an adaptive noise, which reduces the computational complexity brought by the high dimensional and time-variant system in FL.}
    \item \qi{For any pre-processing process such as embedding, we can get the same theoretical guarantee by substituting $\mSigma_{embedding}$ for $\mSigma_{\mD}$ in Eq.~(\ref{eq:data_channel_capacity}), where $\mSigma_{embedding}$ and $\mSigma_{\mD}$ are covariance matrixes calculated by embeddings and data, respectively.}
    \item \qi{In our method, a larger batch size leads to a tighter upper bound and a stronger ability to defend against DRA.}
\end{itemize}

\subsection{Channel Capacity for Existing Methods}

\begin{figure*}[tb]
\begin{center}
\begin{minipage}[b]{0.3\textwidth}
\vskip -0.15in
\centering
\includegraphics[width=\textwidth]{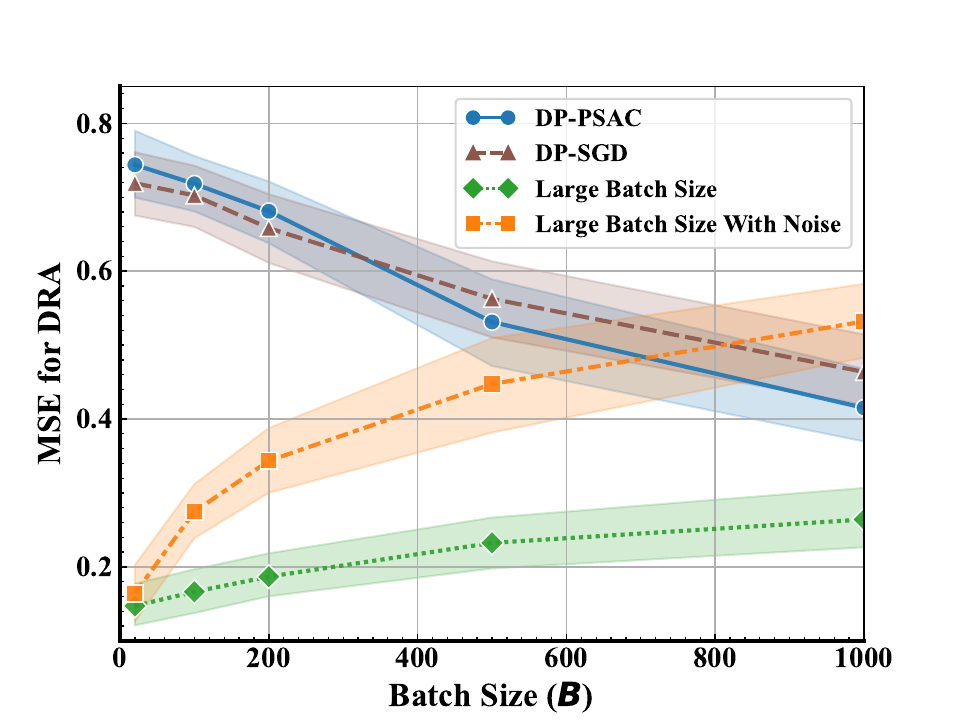}
\caption{\qi{Defensive capabilities of DP and utilizing large batch size according to different batch sizes.}}\label{fig:batch_size_exp}
\end{minipage}
\hfill
\begin{minipage}[b]{0.3\textwidth}
\vskip -0.15in
\centering
\includegraphics[width=\textwidth]{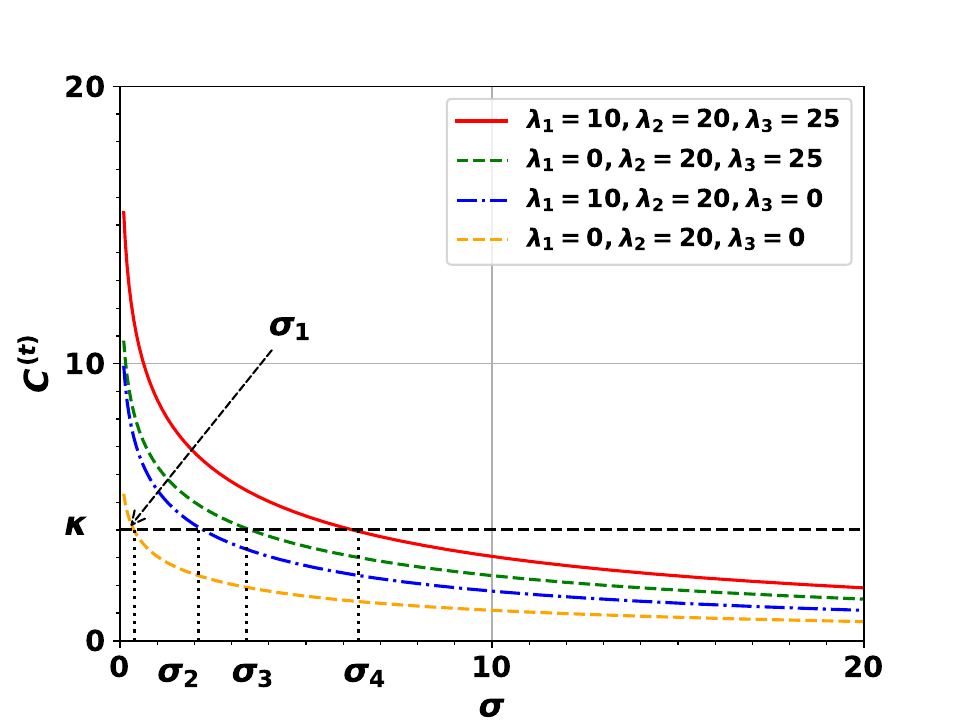}
\caption{Reducing the dimensionality of a gradient is equivalent to adding more noise into the original gradient.}\label{fig:Rel_Compre_Noise}
\end{minipage}
\hfill
\begin{minipage}[b]{0.3\textwidth}
\vskip -0.15in
\centering
\includegraphics[width=\textwidth]{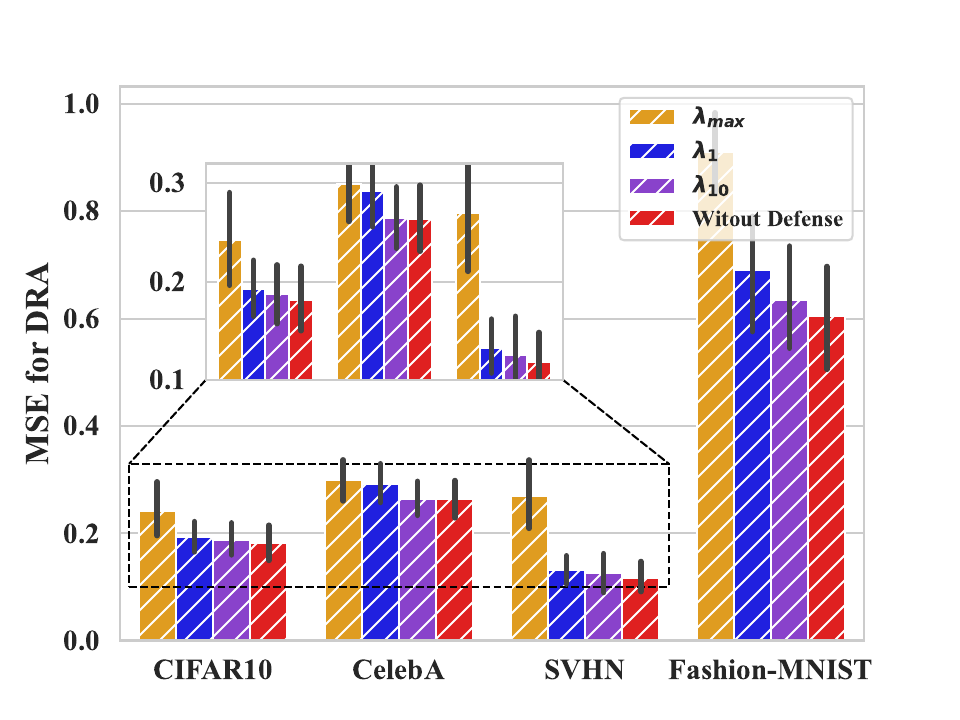}
\caption{\qi{DRA becomes more difficult when we compress the dimension with a large eigenvalue.}}\label{fig:Compression_Eigen}
\end{minipage}
\end{center}
\vskip -0.1in
\end{figure*}

With Thm.~\ref{mse_lower_bound} and Thm.~\ref{thm:channel_capacity}, we can analyze existing methods for defending against DRA
by the channel capacity $C^{(t)}$.

Existing privacy-enhancing methods in FL can be divided into three categories: perturbation, compression, and utilizing large batch size. The method of perturbation is represented by DP. Moreover, we demonstrate that for defending against DRA, the methods of gradient compression and utilizing large batch size are equivalent to adding more noise for perturbation, and the intrinsic mechanism of them are restricting the transmitted information.

\noindent {\bfseries Perturbation.} As a method with theoretical guarantee, DP focuses on the problem of protecting individual information. \qi{In this work, we consider the event-level DP \cite{DBLP:conf/nips/LiuSYK020,DBLP:conf/nips/LevySAKKMS21} because we focus on the privacy issue for a specific client (i.e., the victim). That is, whether the attacker can reconstruct the local dataset with the victim's shared parameters. Specifically, the widely used DP technique in FL is the Gaussian mechanism, thus we analyze the Gaussian mechanism in this section. A typical process of the Gaussian mechanism} for $(\epsilon, \delta)$-DP consists of two stages: gradient clipping, which guarantees the sensitivity of the gradient is bounded; noise addition, which provides the $(\epsilon, \delta)$-DP guarantee based on the bounded sensitivity. To analyze DP based on information theory, we have the following theorem.
\begin{theorem}[Channel capacity for DP]\label{thm:dp_channel_capacity}
    In FL, if the sensitivity of the gradient mapping is upper bounded by $S$, i.e., $\|\vg\|_2 \leq S$, the channel capacity of $(\epsilon, \delta)$-DP, i.e., $C_{DP}$, is upper bounded by following formulas:
    \begin{enumerate}[\quad(1)]
    \setlength{\itemsep}{2pt}
    \setlength{\parsep}{0.5pt}
    \setlength{\parskip}{0pt}
      \item $C_{DP} \leq \frac{B \cdot S^2}{\sigma},$
      \item $C_{DP} \leq \frac{B \cdot \epsilon^2}{2\log(1.25/\delta)},$
    \end{enumerate}
    where $B$ represents the batch size and $\sigma$ is the noise scale.
\end{theorem}

\qi{Compared to the conventional DP theorem, Thm.~\ref{thm:dp_channel_capacity} indicates that batch size $B$ is a key factor of the defense ability to defend against DRA, which has been overlooked by prior literature. Specifically, increasing $B$ reduces DP's ability to defend against DRA, which has been validated by experiments in Fig.~\ref{fig:batch_size_exp} (the details are explained in Appendix~\ref{sec:exp_validating})}.

\noindent {\bfseries Compression.} Another defense technique is compression, it intuitively focuses on reducing transmitted information by reducing the dimension of shared parameters. However, most of these methods lack theoretical guarantees.

Specifically, information is contained in each dimension of the gradient (i.e., the parameter), and the channel capacity decreases accordingly when we reduce the dimension of it. If we analyze the compression in the eigenspace, dimension reduction can be theoretically described by Thm.~\ref{thm:channel_capacity}. \qi{Specifically, the total channel capacity is $C^{(t)}=\frac{1}{2}\sum_{i=1}^{d}\ln\frac{\lambda^{(t)}_i + \sigma}{\sigma}$, and the general term $\frac{1}{2}\ln\frac{\lambda^{(t)}_i + \sigma}{\sigma}$ represents the channel capacity of the $i$-th dimension. Compressing the $i$-th dimension in the eigenspace, i.e., setting $\lambda^{(t)}_i=0$, leads the channel capacity of $i$-th dimension to be $0$, thereby reducing the total channel capacity. As illustrated in Fig.~\ref{fig:Rel_Compre_Noise}, when we compress the gradient, i.e., setting eigenvalues to be $0$, the results are equivalent to adding more noise for perturbation. 
As displayed in Fig.~\ref{fig:Compression_Eigen}, we also conduct experiments to validate our theoretical result, the experiments indicate that the improvement of the defense ability becomes more significant when we compress the dimension with a larger eigenvalue (the details are explained in Appendix~\ref{sec:exp_validating}).} Due to the space limitation, we put the detailed analysis of compression in Appendix~\ref{subsec:more_existing_method}.

\noindent{\bfseries Large Batch Size.} \qi{Another defense strategy for privacy protection in FL is utilizing large batch size \cite{DBLP:conf/nips/ZhuLH19}. Our model can theoretically formalize this strategy. If we denote the batch size as $B$, then the gradient in mini-batch SGD is $ \frac{1}{B}\sum_{i=1}^{B}\nabla_{\mW}F(\mW^{(t)}_{i};\; \mD_i)$, where $\{\mD_i\}_{i=1}^{B}$ represents the set of iid data points sampled from the dataset. The covariance matrix with large batch size is scaled by $B$, i.e., $\mSigma^{(t)}_B=\frac{1}{B}\mSigma^{(t)}$. Hence, with substituting $\mSigma^{(t)}_B$ for $\mSigma^{(t)}$ in Thm.~\ref{thm:channel_capacity}, we have}
\begin{equation}\label{thm:large_batch_size}
    C^{(t)}_B = f_B^{(t)}(\sigma) := \frac{1}{2}\sum_{i=1}^{d}\ln\frac{(\lambda^{(t)}_i/B + \sigma)}{\sigma}. 
\end{equation}
\qi{With the constant noise addition, $C^{(t)}$ is a decreasing function of $B$, which means we can enhance the ability to defend against DRA by increasing $B$. Fig.~\ref{fig:batch_size_exp} experimentally validates this theory and we put the details of Fig.~\ref{fig:batch_size_exp} in Appendix~\ref{sec:exp_validating} due to the space limitation.}

\noindent{\bfseries Guidelines for hyper-parameters.} \qi{Finally, to improve existing defensive algorithms, we provide several guidelines for choosing the hyper-parameters in FL:}
\begin{itemize}
\setlength{\itemsep}{2pt}
\setlength{\parsep}{0.5pt}
\setlength{\parskip}{0pt}
    \item \qi{A smaller batch size in DP algorithm is more effective to defend against DRA.}
    \item \qi{Compressing the dimension with larger eigenvalue results in the stronger ability to defend against DRA.}
    \item \qi{A larger batch size leads to a stronger defense ability to defend against DRA when we add a constant noise to the parameter during training.}
\end{itemize}

\section{Experiment}\label{sec:exp}

\begin{figure*}[htb]
\begin{center}
\begin{minipage}[b]{0.9\textwidth}
\includegraphics[width=\textwidth]{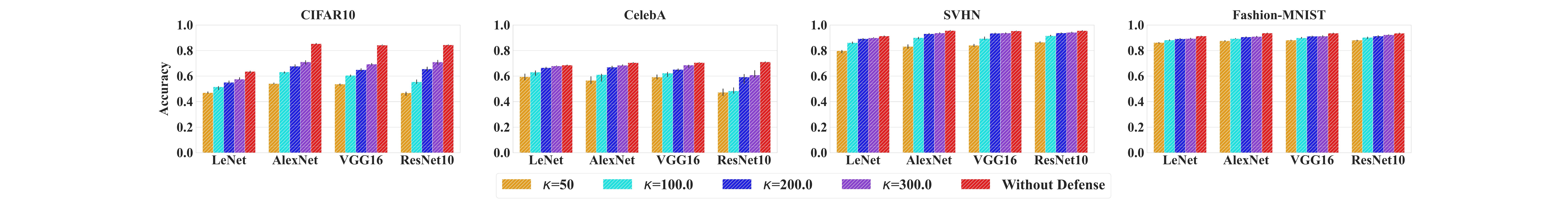}
\vskip -0.1in
\end{minipage}
\vskip -0.1in
\caption{The effect of channel capacities ($\kappa$) for model accuracy (Natural Channel).}\label{fig:ModelACC_ChannelCapacity}
\begin{minipage}[b]{0.9\textwidth}
\includegraphics[width=\textwidth]{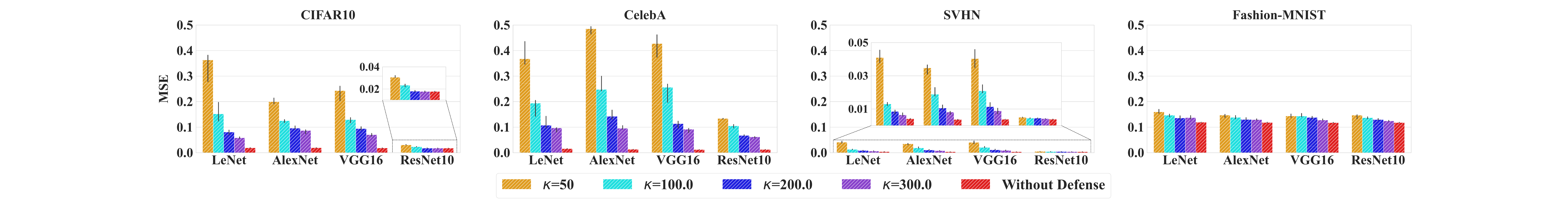}
\vskip -0.1in
\end{minipage}
\vskip -0.1in
\caption{The effect of channel capacities ($\kappa$) for DRA (Natural Channel).}\label{fig:ModelInv_ChannelCapacity}
\vskip -0.2in
\end{center}
\end{figure*}

\begin{figure*}
\begin{center}
\begin{minipage}[b]{0.9\textwidth}
\vskip -0.3in
\includegraphics[width=\textwidth]{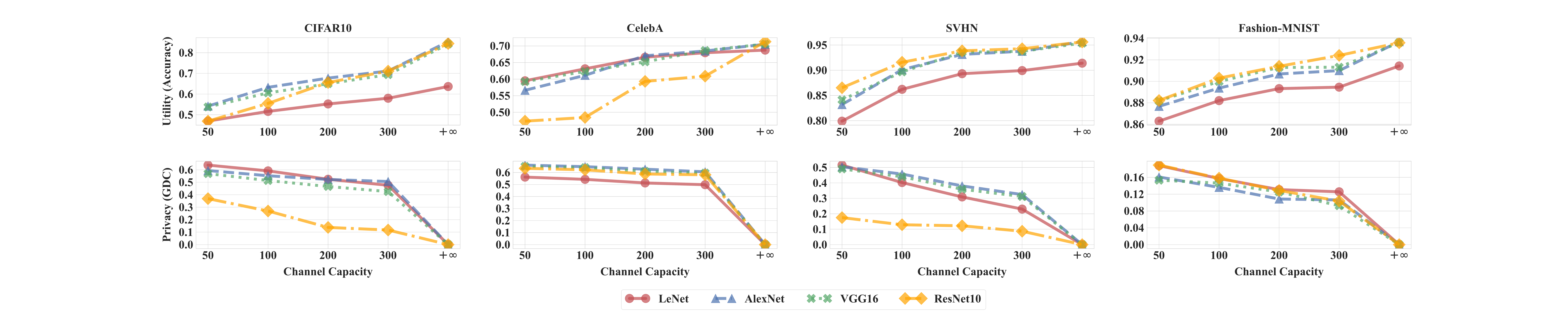}
\vskip -0.1in
\end{minipage}
\vskip -0.1in
\caption{Utility privacy tradeoff according to different channel capacities.}\label{fig:Utility_Privacy_Tradeoff}
\end{center}
\vskip -0.2in
\end{figure*}

In this section, we conduct experiments with various models and datasets to validate our theories and compare our methods with other privacy-enhancing techniques. All experiments are performed upon a Supermicro SYS-420GP-TNR server with two Intel(R) Xeon(R) Gold 6348 CPUs (2$\times$28 cores), Ubuntu 18.04.1, 10GB memory, and four NVIDIA A100 PCIe 80GB GPUs. Meanwhile, to eliminate the impact of randomness, each experiment is repeated 10 times.

\subsection{Experimental Settings}

\noindent{\bfseries Datasets.} Dataset is a task-dependent factor, which means we cannot change the dataset when the training task is decided. In our experiments, we resize the data to 32*32 for comparison, and utilize four classical datasets, including CIFAR10 \cite{2012Learning}, CelebA \cite{DBLP:conf/iccv/LiuLWT15}, SVHN \cite{2011Reading}, and Fashion-MNIST \cite{DBLP:journals/corr/abs-1708-07747}.





\begin{figure*}[htb]
\vskip 0.1in
\begin{center}
\begin{minipage}[b]{0.9\textwidth}
\includegraphics[width=\textwidth]{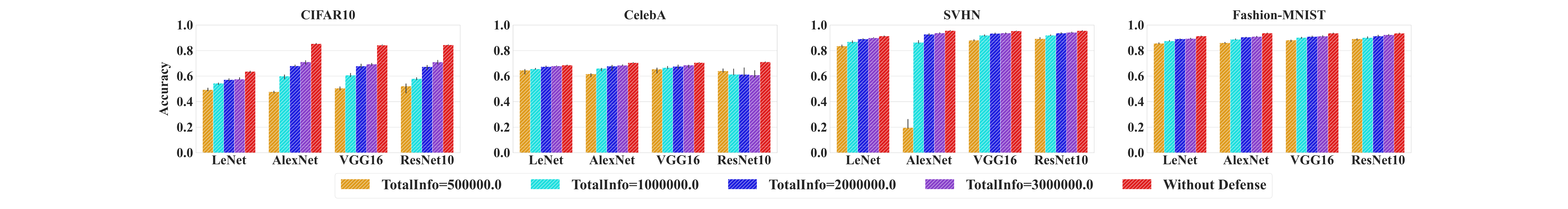}
\vskip -0.1in
\end{minipage}
\vskip -0.1in
\caption{The effect of optimization number ($n$) for model accuracy when $\kappa=300$ (Natural Channel).}\label{fig:ModelACC_Number}
\vskip 0.1in
\begin{minipage}[b]{0.9\textwidth}
\includegraphics[width=\textwidth]{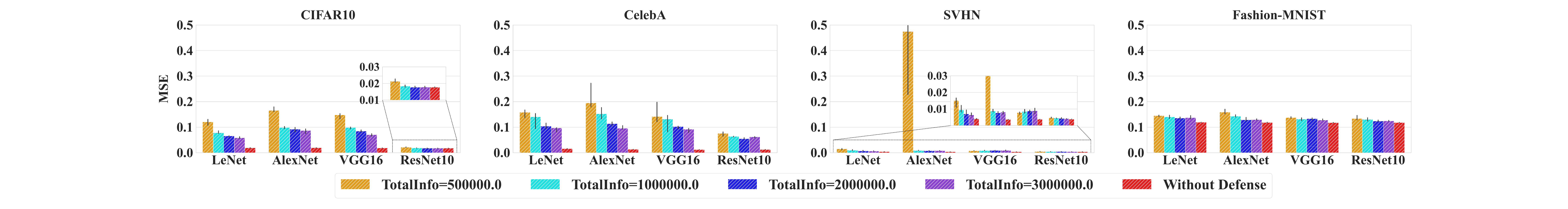}
\vskip -0.1in
\end{minipage}
\vskip -0.1in
\caption{The effect of optimization number ($n$) for DRA when $\kappa=300$ (Natural Channel).}\label{fig:ModelInv_Number}
\begin{minipage}[b]{0.9\textwidth}
\vskip 0.1in
\includegraphics[width=\textwidth]{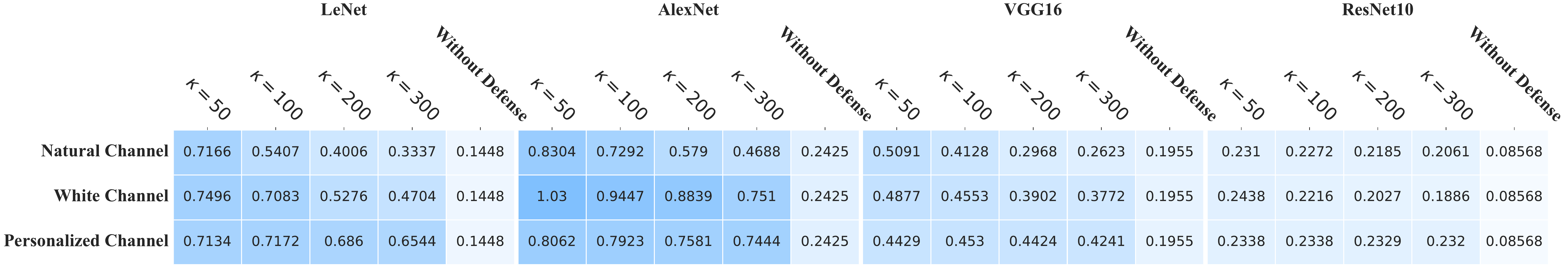}
\vskip -0.1in
\end{minipage}
\caption{Heatmaps of MSE for the gradient inversion attacks on CIFAR-10.}\label{fig:Heat_map_MSE}
\end{center}
\vskip -0.1in
\end{figure*}

\noindent{\bfseries Basic Models.} In this work, we experiment with four classical model architectures, including LeNet \cite{DBLP:journals/pieee/LeCunBBH98}, AlexNet \cite{DBLP:conf/nips/KrizhevskySH12}, VGG16 \cite{DBLP:journals/corr/SimonyanZ14a}, and ResNet10 \cite{DBLP:conf/cvpr/HeZRS16}. Moreover, model architectures are task-independent, we can select architectures according to different goals, e.g., utility or privacy.

\noindent{\bfseries Attacks.} We test our methods on two typical attacks in FL: \emph{MIA} and \emph{DRA}, the details of different attacks are as follows.

\emph{Membership Inference Attack.} For MIA, we use the white-box attack~\cite{DBLP:conf/sp/NasrSH19, DBLP:conf/uss/LiuWH000CF022}. Meanwhile, we employ a partial knowledge attacker, which means the attacker can access to part of the training dataset. In this case, the attacker has much stronger background knowledge. Moreover, we use four inputs for attacking~\cite{DBLP:conf/uss/LiuWH000CF022}: the samples' ranked posteriors, classification loss, gradients of the last layer, and one-hot encoding of the true label. These inputs are fed into different neural networks to get different embeddings, then we concatenate all embeddings as the input of a 4-layer MLP to get the inference.

\emph{Data Reconstruction Attack.} For DRA, we employ two representative attacks: model inversion attack and gradient inversion attack. For model inversion attack~\cite{2015Model,DBLP:conf/uss/LiuWH000CF022}, we first construct a dummy input with auxiliary information (i.e., the mean value of images that are not in the training dataset) as the input for the target model, then utilize different target labels to optimize the dummy input. We use Adam optimizer with a learning rate of 1e-2 for 600 iterations. Additionally, we employ the settings proposed by Geiping et al.~\cite{DBLP:conf/nips/GeipingBD020} to investigate the defense ability for gradient inversion attacks.


\noindent{\bfseries Metric.} We evaluate our design with various criteria, including the model utility, the defense capability against different attacks, the utility-privacy trade-offs, and the efficiency. To this end, we employ the following metrics for evaluation.

\emph{Test Accuracy.} FL searches for an accurate model for classification, hence we use test accuracy as the metric for utility.

\emph{AUC.} We use the attack AUC to measure the attack performance for MIA. \qi{It's worth noting that a smaller AUC means a stronger defense capability.}

\emph{MSE.} \qi{We use MSE as the main metric for DRA since MSE is a general metric that indicates the convergence of random variables. That is, if the MSE of two random variables is 0, we conclude that they have identical distributions \cite{wasserman2004all}, which means a perfect reconstruction. Additionally, a large MSE means a stronger ability to defend against DRA. Moreover, we also utilize PSNR, SSIM, and Cosine Similarity to evaluate the defense ability, we put these results in Appendix~\ref{sec:appendix_exp} due to the space limitation.} Specifically, for model inversion attacks, we calculate the metric between the reversed data and the center of the corresponding class for different classes. We use the median of all classes as the final metric. For gradient inversion attacks, we employ the mean value of the metric between the reconstructed data and the target data.

\emph{General Defense Capability (GDC).} For evaluating the comprehensive defense capability, we use the improvements of the aforementioned attacks. Specifically, for MIA, we define the improvement as
$$ IMP_{MIA} =({AUC}_{without\_def} - {AUC})/{AUC}_{without\_def}.$$
While the improvement for DRA is
$$IMP_{DRA} = (\frac{1}{MSE_{without\_def}} - \frac{1}{MSE})/\frac{1}{{MSE}_{without\_def}}.$$
Finally, we define general defense capability as
$$GDC=(IMP_{MIA} + IMP_{DRA})/{2},$$
and a larger $GDC$ means a stronger defense capability.

\noindent{\bfseries Default parameter configuration.} \qi{In our experiments, we mainly utilize natural channel to investigate the channel parameters. Specifically, we fix the optimization rounds as $n=1\times 10^4$ and utilize $\kappa$ in \{50, 100, 200, 300\} to investigate the effect of different channel capacities. Additionally, we fix $\kappa=300$ and utilize $TotalInfo$ in $\{5\times10^5,  1\times10^6, 2\times10^6, 3\times10^6\}$ to investigate the effect of different $n$, where $n=\lfloor \frac{TotalInfo}{\kappa} \rfloor$ and $\lfloor \cdot \rfloor$ is the floor function. We will explicitly explain it when we utilize different configurations.}

\subsection{The Effect of Controlled Channel}


\begin{figure*}[htb]
\begin{center}
\begin{minipage}[b]{0.9\textwidth}
\includegraphics[width=\textwidth]{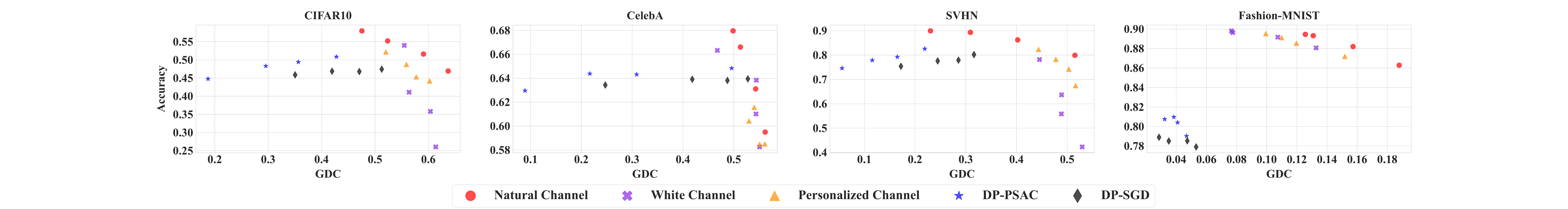}
\vskip -0.1in
\end{minipage}
\caption{Evaluating different methods in the Utility-Privacy plane.}\label{fig:Tradeoff_in_Utility_Privacy_Plane}
\begin{minipage}[b]{0.9\textwidth}
\includegraphics[width=\textwidth]{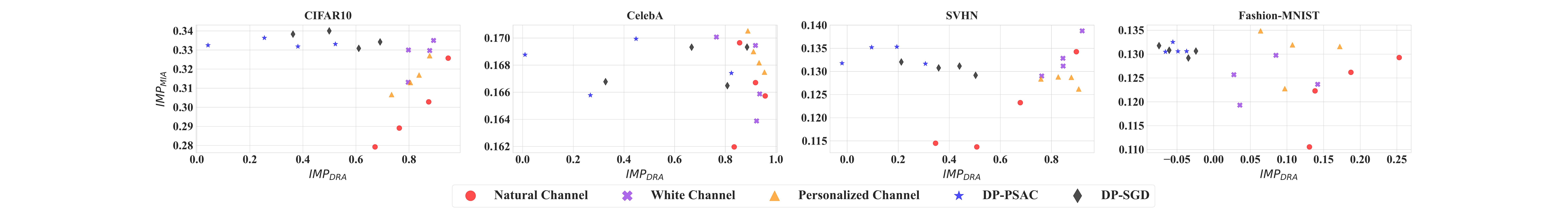}
\vskip -0.1in
\end{minipage}
\caption{Evaluating different methods in the Defense-Defense plane.}\label{fig:Tradeoff_in_Defense_Plane}
\vskip -0.2in
\end{center}
\vskip -0.2in
\end{figure*}

\noindent{\bfseries Impact of channel capacity ($\kappa$).} $\kappa$ represents the upper bound of information leakage in a single training round. 
As illustrated in Fig.~\ref{fig:ModelACC_ChannelCapacity}, when we enhance the channel, i.e., increase $\kappa$, the accuracy increases accordingly regardless of the datasets and the model architectures, which means a wider channel leads to a better utility. However, the increasing utility is at the cost of reducing data privacy. 
Fig.~\ref{fig:ModelInv_ChannelCapacity} displays the corresponding defense capability against DRA. Specifically,
for DRA attacks, the MSE decreases according to the increasing $\kappa$, which means the reconstructed data is closer to the target data. These results indicate that in the practical scenario, we can adjust channel capacity by $\kappa$ to balance the utility and the privacy according to various requirements.

\qi{Moreover, we explicitly display the utility-privacy tradeoff in Fig.~\ref{fig:Utility_Privacy_Tradeoff}. Specifically, the utility increases with the increase of channel capacity $\kappa$, while the ability of privacy protection decreases with the increase of $\kappa$. These results are reasonable since channel capacity constrains the transmitted information, and the reconstruction error, i.e., the MSE, can be theoretically restricted when the transmitted information is limited based on Thm.~\ref{mse_lower_bound}. Meanwhile, more available information results in a more accurate model, indicating that the adjustable channel capacity can be utilized to balance the utility-privacy tradeoff.}

\noindent{\bfseries Impact of optimization rounds ($n$).} Different from the channel capacity $C$, the number of optimization rounds $n$ affects the parameter channel through another dimension: information accumulation. 
Similarly, the results in Fig.~\ref{fig:ModelACC_Number} 
and \ref{fig:ModelInv_Number} imply that $n$ is another parameter to influence the utility and the defense ability of the controlled channel. Increasing $n$ results in obtaining more information from the local dataset, thereby enhancing the utility while reducing the defense ability.

\noindent{\bfseries Defense ability against gradient inversion attack.} We randomly select 30 images from each dataset as the target data and evaluate the defense abilities to defend against gradient inversion attacks. As shown in Fig.~\ref{fig:Heat_map_MSE}, the MSE of data reconstruction consistently increases as we decrease the threshold $\kappa$, which means a smaller $\kappa$ leads to a stronger ability to defend against DRA. The MSE of the White Channel and the Personalized Channel are larger than the Natural Channel, indicating that the other two methods provide stronger defense capability compared to the Natural Channel. 

\subsection{Comparing with Other Methods}
In this section, we compare our methods with two classic perturbation mechanisms: DP-SGD \cite{DBLP:conf/ccs/AbadiCGMMT016} and DP-PSAC \cite{DBLP:journals/corr/abs-2212-00328}. Specifically, DP-SGD first utilizes DP for ML and DP-PSAC is the state-of-art method to improve DP-SGD by adaptively clipping the gradients.

\noindent{\bfseries Utility-Privacy trade-off.} We first compare our methods with DP-SGD and DP-PSAC in the utility-privacy plane. The plane is formed by two axes: the $x$-axis represents GDC, which is the general ability for privacy protection. While the $y$-axis represents the accuracy, which represents the utility. For all methods, we use the model of LeNet and fix the rounds as $n=1\times10^4$. The results are displayed in Fig.~\ref{fig:Tradeoff_in_Utility_Privacy_Plane}. Specifically, in our methods, we use channel capacity in \{50, 100, 200, 300\} for the Natural Channel, and \{500, 800, 1000, 1500\} for the White Channel and the Personalized Channel.
Additionally, we use 1:50 in $\vbeta$ for the Personalized Channel to protect the attributes around the eyes (Appendix~\ref{sec:prior_knowledge} indicates the details of $\vbeta$). For DP-SGD and DP-PSAC, we use the clipping bound $S=1.0$, $\delta=1\times 10^{-5}$ ($\delta=3\times 10^{-6}$ for CelebA), and the multiplier $\sqrt{\sigma}$ in \{0.8, 0.57, 0.46, 0.2066\}  (the corresponding privacy budget $\epsilon$ for these DP training are \{1.705, 5.120, 12.026, 331.668\}, respectively) 
. \qi{According to Thm.~\ref{thm:dp_channel_capacity}, these $\sigma$ result in channel capacities in \{100, 200, 300, 1500\} when we utilize $B=64$, which are identical to the channel capacities of our methods.} As in Fig.~\ref{fig:Tradeoff_in_Utility_Privacy_Plane}, the Natural Channel achieves the best utility-privacy trade-off. The reason is that it utilizes the original importance to maintain the information in data attributes, and the constrained channel capacity ensures the ability to defend against DRA.

\noindent{\bfseries Comparison of the details in defense ability.} To investigate the detailed defense ability of different methods, we decouple the defense capabilities to display them in the defense-defense plane. Specifically, $x$-axis represents $IMP_{DRA}$ and $y$-axis represents $IMP_{MIA}$. As displayed in Fig.~\ref{fig:Tradeoff_in_Defense_Plane}, DP specializes in defending against MIA but is weaker in defending against DRA, this is reasonable since DP focuses on protecting individual information. Compared to DP, our methods specialize in defending against DRA. Moreover, our methods achieve the best comprehensive defense capability for both attacks.

\begin{table}[htb]
\footnotesize
\renewcommand{\arraystretch}{1.4}
\setlength{\tabcolsep}{2.5pt}
\vskip -0.1in
\caption{Training time (s) for 25 epochs of different models}\label{tb:compare_runtime_data_parameter_space}
\begin{center}
\begin{tabular}{lcccc}
\toprule
Model & Parameters  & SGD &  DP-SGD & Ours\\
\midrule
 \multirow{2}{*}{LeNet}                     & \multirow{2}{*}{$6.20\times 10^4$}      & \multirow{2}{*}{283.43$\pm$0.95}       &  445.45$\pm$20.41  & 312.99$\pm$1.40\\
 ~& ~& ~ & (+57.16\%)& (+10.43\%) \\
 \hline
 \multirow{2}{*}{ResNet10}         & \multirow{2}{*}{$4.90\times 10^6$}     & \multirow{2}{*}{624.23$\pm$8.66}       & 1397.34$\pm$17.38 & 651.29$\pm$10.39\\
 ~& ~& ~ & (+123.85\%) & (+4.33\%)\\
 \hline
 \multirow{2}{*}{AlexNet}             & \multirow{2}{*}{$3.59\times 10^7$}      & \multirow{2}{*}{374.13$\pm$2.68}      &1023.83$\pm$2.85 & 406.02$\pm$1.04\\
~& ~& ~ & (+173.66\%) & (+8.52\%)\\
\hline
 \multirow{2}{*}{VGG16}             & \multirow{2}{*}{$1.34\times 10^8$}  &    \multirow{2}{*}{581.69$\pm$5.35}      & 3402.91$\pm$4.72 & 612.99$\pm$4.62 \\
 ~& ~& ~ & (+485.00\%) & (+5.38\%)\\
\bottomrule
\end{tabular}
\end{center}
\vskip -0.1in
\end{table}

\noindent{\bfseries Efficiency of constraining in the data space.} As displayed in Tab.~\ref{tb:compare_runtime_data_parameter_space}, we compare the efficiency of our method with DP-SGD, which is a method that works in the parameter space. We use OPACUS 1.1.2 \cite{DBLP:journals/corr/abs-2109-12298} for the implementation of DP-SGD. The results indicate that when we transform the operations to the data space, it significantly reduces the amount of calculation, i.e., the training time, especially when the dimension of the parameter increases.

\section{Related Work}

\noindent{\bfseries Differential Privacy.} Differential Privacy \cite{DBLP:conf/icalp/Dwork06} is an important technique for privacy protection. For the definition of differential privacy, Mironov~\cite{DBLP:conf/csfw/Mironov17} proposes to utilize R\'{e}nyi divergence, which leads to compact and accurate privacy loss. Moreover, Abadi et al.~\cite{DBLP:conf/ccs/AbadiCGMMT016} propose an empirical algorithm DP-SGD for applying DP to ML training, then multiple literature tries to improve the performance of DP-SGD~\cite{DBLP:conf/nips/AgarwalSYKM18, DBLP:conf/nips/AndrewTMR21, DBLP:journals/corr/abs-2206-07136, DBLP:journals/corr/abs-2212-00328}. Most of them focus on the utility, trying to adaptively clip the gradients or add suitable perturbations. For the application of DP, Levy et al.~\cite{DBLP:conf/nips/LevySAKKMS21} propose to protect user-level DP instead of ensuring the privacy of individual samples, which makes DP more reliable to FL. Truex et al.~\cite{DBLP:conf/eurosys/Truex0CGW20} present a protocol LDP-Fed to formally ensure data privacy in collecting local parameters with high precision. However, DP is still vulnerable to DRA, several researchers indicate that they can reconstruct the data without violating the requirements of DP~\cite{DBLP:conf/kdd/Cormode11, DBLP:journals/corr/abs-2211-03128}, implying that the algorithms are still vulnerable to DRA. Compared to DP, our work aims to defend against DRA based on information theory.

\noindent{\bfseries Limiting Mutual Information.} Another related concept of privacy protection depends on MI. Several studies connect MI with DP by deriving the upper bound of MI for a distinct DP mechanism~\cite{DBLP:conf/csfw/BartheK11,DBLP:conf/ifip1-7/AlvimACDP11,DBLP:conf/ccs/CuffY16}. For privacy protection, Li et al.~\cite{DBLP:conf/kdd/LiDYC020} propose to train a feature extractor that minimizes MI between the output features and the assigned label while maximizing the MI between output features and the original data. Similarly, Osia et al.~\cite{DBLP:journals/tkde/OsiaTSKHR20} propose to train an extractor with the variational lower bound for MI estimation. Moreover, Hannun et al.~\cite{DBLP:conf/ijcai/HannunGM22} utilize Fisher Information to measure the information leakage, and the conclusions are consistent with our theories. Finally, we can also model FL based on information theory \cite{DBLP:journals/corr/abs-1911-07652, DBLP:journals/tpds/UddinXLYG21}, which measures the MI between different variables in FL. Compared to the former methods, our technique models the communication channel of FL directly based on information theory and utilizes its upper bound to derive practical methods for constraining the transmitted information, which leads to the strong ability to defend against DRA.

\section{Discussion and Conclusion}
\qi{Our method constrains the information leakage of the black-box model according to an upper bound derived by maximum entropy distribution (Lemma~\ref{lemma:upper_bound_entropy}). We can tighten the upper bound by utilizing a large batch size or incorporating domain knowledge. Among them, utilizing large batch size for training causes the output distribution to be closer to a Gaussian distribution, 
while incorporating domain knowledge enables us to get more properties of the output distribution. If we get a tighter upper bound, the tradeoff between utility and privacy can be further improved.}


In summary, as the reconstruction error of DRA is decided by the transmitted information, we build a channel model to measure the information leakage of the black-box model in FL. The model indicates that the amount of transmitted information is decided by the channel capacity $C$ and the number of optimization rounds $n$. Guided by the model, we develop methods to constrain the channel capacity within a threshold $\kappa$. Combining it with the limited optimization rounds $n$, the upper bound of the total transmitted information remains below $n\cdot \kappa$, which ensures the ability to defend against DRA. Furthermore, we transform the operations of constraining channel capacity from the parameter space to the data space. The transformation significantly improves the training efficiency and the model accuracy under constrained information leakage.
Finally, extensive experiments with real-world datasets validate the benefit of our methods.

\section{Acknowledgments}

We thank our shepherd and anonymous reviewers for their
thoughtful comments. This work was supported in part by the National Science Foundation for Distinguished Young Scholars of China under No. 61825204, National Natural Science Foundation of China under No. 62202258, No. 62132011, No. 61932016, Beijing Outstanding Young Scientist Program under No. BJJWZYJH01201910003011, China Postdoctoral Science Foundation under No. 2021M701894, China National Postdoctoral Program for Innovative Talents, Shuimu Tsinghua Scholar Program, Lenovo Young Scientist Program, and the Beijing National Research Center for Information Science and Technology key projects. Yi Zhao and Ke Xu are the corresponding authors.

\bibliographystyle{plain}
\bibliography{PreventingDRA}

\appendix

\section{The detail explanation of DRA}\label{sec:detail_dra}

\qi{For DRA in FL, the attacker aims to build a reconstruction $\hat{\mD}(\mW_i, \mW_o)$ to approximate the data $\mD$, where $\mW_i$ and $\mW_o$ are the transmitted parameters in FL. Moreover, if the MSE between $\hat{\mD}(\mW_i, \mW_o)$ and $\mD$ achieves $0$, i.e., $\mathbb{E}\|\hat{\mD}(\mW_i, \mW_o) - \mD\|^2=0$, we conclude that $\hat{\mD}(\mW_i, \mW_o)$ and $\mD$ have identical distributions \cite{wasserman2004all}.}

\qi{Additionally, with the definition of $\mD$, i.e., the random variable that follows the distribution of the local dataset, the MIA  \cite{DBLP:conf/sp/CarliniCN0TT22} can be viewed as inferring whether a particular data point belongs to the support set of $\mD$.}

\section{Details of the Prior Knowledge}\label{sec:prior_knowledge}
\begin{figure}[htb]
\begin{center}
\begin{minipage}[b]{0.36\textwidth}
\vskip -0.1in
\centering
\includegraphics[width=\textwidth]
{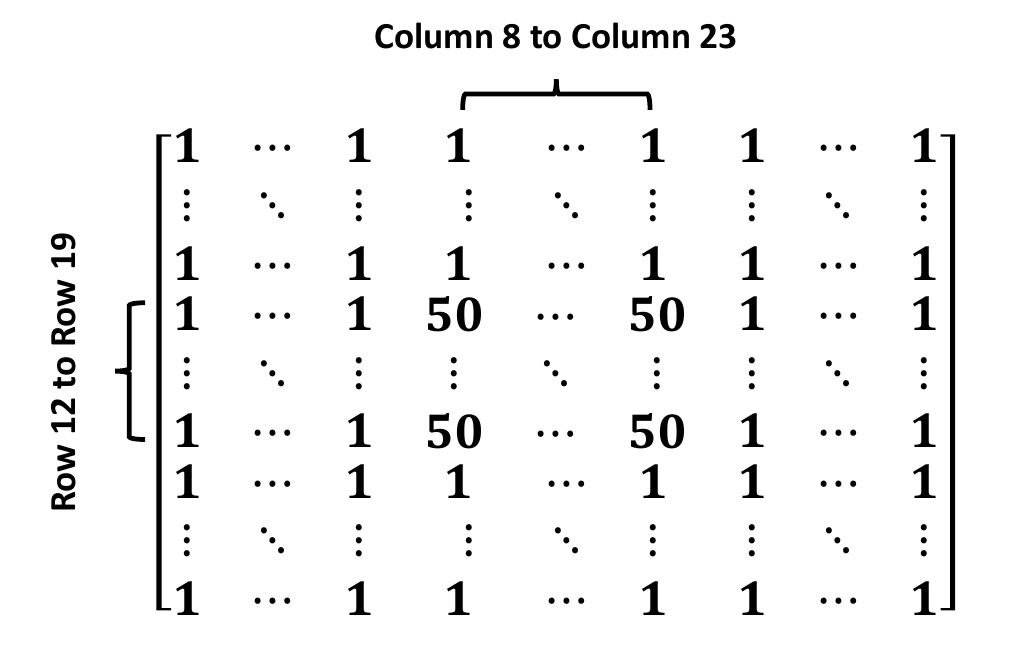}
\end{minipage}
\vskip -0.1in
\caption{\qi{The prior knowledge $\bm{\beta}$ used for the Personalized Channel in Sec.~\ref{sec:exp}. Specifically, the number in the matrix represents the coefficient for the added noise.}}\label{fig:Prior_Knowledge_Beta}
\end{center}
\vskip -0.1in
\end{figure}

\qi{For the Personalized Channel, the prior knowledge $\bm{\beta}$ used in Sec.~\ref{sec:exp} is a $32*32*c$ tensor, where $c$ represents the number of image channel (e.g., $c=1$ for Fashion-MNIST dataset and $c=3$ for CIFAR-10 dataset, respectively). The number in $\bm{\beta}$ represents the coefficient for the added noise, which means we add more noise to the dimension if the coefficient of this dimension is larger. In these experiments, we employ position-based prior knowledge, which means the prior knowledge for all image channels are identical. Therefore, the important parameters are the numbers of one channel, which is a $32*32$ matrix (as displayed in Fig.~\ref{fig:Prior_Knowledge_Beta}).}

\qi{Specifically, as we aim to protect the private information around eyes for the data in CelebA (as displayed in Fig.~\ref{fig:Overview_Channel_Capacity}), we set the number of these positions (i.e., the sub-matrix with row 12 to 19 and column 8 to 23.) to be $50$, and set the rest number of the matrix to be $1$. Therefore, this prior knowledge adds $50$ times noise to the chosen positions (i.e., the positions around the eyes) to provide stronger privacy protection for this area. Additionally, we also use this prior knowledge for other datasets, including Fashion-MNIST, CIFAR10, and SVHN.}

\section{Proofs of Lemmas and Theorems}
\subsection{Proof of Theorem~\ref{mse_lower_bound}}

Before the proof of Thm.~\ref{mse_lower_bound}, we have the following lemma.

\begin{lemma}\label{lemma:bound_determinant}
\vskip -0.05in
    For any d-dimensional semi-positive definite matrix $\mA$, i.e., $\mA \in \mathbb{R}^{d\times d}$, we have 
    \begin{equation}
    \setlength{\abovedisplayskip}{5pt}
  \setlength{\belowdisplayskip}{5pt}
        \det (\mA) \leq \left(\frac{\tr (\mA)}{d}\right)^d.
    \end{equation}
\vskip -0.05in
\end{lemma}
\begin{proof}
    Since $\mA$ is a semi-positive definite matrix, we have its eigen values, i.e., $\{\lambda_i\}^d_{i=1}$, are non-negative.
    we can get
    \begin{align*}
        \det(\mA) = \prod_{i=1}^{d} \lambda_i \leq \left(\frac{\sum_{i=1}^{d} \lambda_i}{d}\right)^d = \left(\frac{\tr (\mA)}{d}\right)^d,
    \end{align*}
    where the inequality depends on the AM-GM inequality.
\end{proof}
\vskip -0.05in
Then based on Lemma~\ref{lemma:bound_determinant}, we can prove Thm.~\ref{mse_lower_bound}.
\begin{proof}
    As $Cov(\mD)$ is the covariance matrix of $\mD$, we have
    \begin{align*}
        h(\mD | \mW) & \overset{(1)}{\leq} \mathbb{E}_{\mW} [\frac{d}{2} \log(2\pi e) + \frac{1}{2} \log \det(Cov(\mD|\mW))] \\
                         & \overset{(2)}{\leq} \frac{d}{2} \log(2\pi e) + \mathbb{E} [ \frac{d}{2} \log (\frac{\tr(Cov(\mD|\mW))}{d})] \\
                         & =\frac{d}{2} \log(2\pi e) + \mathbb{E} [ \frac{d}{2} \log (\frac{\mathbb{E}[\|\mD - \mathbb{E}[\mD]\|^2|\mW]}{d})] \\
                         & \overset{(3)}{\leq} \frac{d}{2} \log(2\pi e) + \mathbb{E} [ \frac{d}{2} \log (\frac{\mathbb{E}[\|\mD - \hat{\mD}(\mW)\|^2|\mW]}{d})] \\
                         & \overset{(4)}{\leq} \frac{d}{2} \log(2\pi e) +  \frac{d}{2} \log (\frac{\mathbb{E}[\mathbb{E}[\|\mD - \hat{\mD}(\mW)\|^2|\mW]]}{d}) \\
                         & = \frac{d}{2} \log(2\pi e) +  \frac{d}{2} \log (\frac{\mathbb{E}[\|\mD - \hat{\mD}(\mW)\|^2]}{d}),
    \end{align*}
    where (1) is a consequence that with identical mean vector and covariance matrix, Gaussian distribution is the maximum entropy distribution. (2) depends on Lemma~\ref{lemma:bound_determinant}. The reason for (3) is that mean vector is the optimal estimator in terms of mean squared error, and (4) is a consequence of Jensen's inequality. Hence,
    \begin{align*}
        \mathbb{E}[\|\mD - \hat{\mD}(\mW)\|^2/d]) & \geq \frac{1}{2 \pi e} e^{2h(\mD | \mW)/d} \\
        &=  \frac{e^{2h(\mD)/d}}{2 \pi e} e^{-2I(\mD; \mW)/d},
    \end{align*}
    where the last equality is a consequence of Eq.~(\ref{eq:mutual_info}).
\end{proof}

\subsection{Proof of Lemma~\ref{lemma:upper_bound_entropy}}
\begin{proof}
  Firstly, as $\mX$ and $\mY$ are independent, then $\mathE[\mX+\mY] = \vmu_{\mX} + \vmu_{\mY}$ and $Cov(\mX+\mY) = \mSigma_{\mX} + \mSigma_{\mY}$.

  Secondly, with a specific mean vector and covariance matrix, the maximum entropy distribution is Gaussian, which implies that $h(\mX+\mY) \leq h(\mathN(\vmu_{\mX} + \vmu_{\mY}, \; \mSigma_{\mX} + \mSigma_{\mY}))$.

  Finally, if $\mX \sim \mathN(\vmu_{\mX}, \;\mSigma_{\mX})$, then $\mX+\mY \sim \mathN(\vmu_{\mX} + \vmu_{\mY}, \; \mSigma_{\mX} + \mSigma_{\mY})$, which concludes the proof.
\end{proof}

\subsection{Proof of Theorem~\ref{thm:channel_capacity}}
\begin{proof}
  According to the relationship between mutual information and differential entropy, we have
  {\setlength{\abovedisplayskip}{5pt}
  \setlength{\belowdisplayskip}{5pt}
  \begin{align}
    & \; I(\mD, \; \widetilde{\mW}^{(t)}_{o} \;|\; \mW^{(t)}_{i} ) \notag\\
    =& \;  h(\widetilde{\mW}^{(t)}_{o} \;| \; \mW^{(t)}_{i}) - h(\widetilde{\mW}^{(t)}_{o} \;|\; \mD, \; \mW^{(t)}_{i}) \notag\\
                                                                    = &\; h(\widetilde{\mW}^{(t)}_{o} \;| \; \mW^{(t)}_{i}) - h(\mathN(\bm{0},\; \sigma \cdot \rmI)) \notag\\
                                                                    \leq &\; h(\mathN(\vmu^{(t)},\; \mSigma^{(t)} + \sigma \cdot \rmI))-h(\mathN(\bm{0},\; \sigma \cdot \rmI)), \label{upper_bound_gaussian}
  \end{align}
  where the second equality depends on the fact that $\mW^{(t)}_{o}$ is a constant when $\mW^{(t)}_{i}$ and $\mD$ are both observed.
  The last inequality is an immediate consequence of Lemma \ref{lemma:upper_bound_entropy} with $\widetilde{\mW}^{(t)}_{o} = \mW^{(t)}_{o}+\sqrt{\sigma} \cdot \bm{\xi}$.}

  Then with the differential entropy of Gaussian distribution, we can further transform Eq.~(\ref{upper_bound_gaussian}) to
  \begin{align}\label{Gaussian_det}
      &h(\mathN(\vmu^{(t)},\; \mSigma^{(t)} + \sigma \cdot \rmI))-h(\mathN(\bm{0},\; \sigma \cdot \rmI)) \notag\\
       = &\frac{1}{2} \ln (2\pi e)^d \det(\mSigma^{(t)} + \sigma \cdot \rmI) - \frac{1}{2}\ln (2\pi e)^d \sigma^d,
  \end{align}
  where $\det(\cdot)$ denotes the determinant of a matrix and the second term of Eq.~(\ref{Gaussian_det}) is decided by $\det(\sigma \cdot \rmI) = \sigma^d$.

  Next, we focus on the first term of Eq.~(\ref{Gaussian_det}). Note that $\mSigma^{(t)}$ is the covariance matrix of $\mW^{(t)}_{o}$, hence it's a real symmetric matrix, then according to eigen decomposition, we have
  \begin{equation}\label{eq:sigma_decomp}
    \mSigma^{(t)} = \mQ^{(t)} \bm{\Lambda}^{(t)} \mQ^{(t)\mathrm{T}} = \mQ^{(t)} \mathrm{diag}(\lambda^{(t)}_1, \cdots, \lambda^{(t)}_d )\mQ^{(t)\mathrm{T}},
  \end{equation}
  where $\mQ^{(t)}$ is an orthogonal matrix, i.e., $\mQ^{(t)} \mQ^{(t)\mathrm{T}} = \rmI$. Hence, $\sigma \cdot \rmI = \mQ^{(t)}  (\sigma \rmI)  \mQ^{(t)\mathrm{T}} $. Then we have
  \begin{align*}
    & \; \det(\mSigma^{(t)} + \sigma \cdot \rmI) \\
    =&\; \det(\mQ^{(t)} \mathrm{diag}(\lambda^{(t)}_1, \cdots, \lambda^{(t)}_d ) \mQ^{(t)\mathrm{T}} + \mQ^{(t)} (\sigma \rmI) \mQ^{(t)\mathrm{T}}) \\
    =&\; \det(\mathrm{diag}(\lambda^{(t)}_1, \cdots, \lambda^{(t)}_d )+\sigma \rmI) = \prod_{i=1}^{d}(\lambda^{(t)}_i + \sigma).
  \end{align*}
  Then we have $I(\mD, \; \widetilde{\mW}^{(t)}_{o} \;|\; \mW^{(t)}_{i} ) \leq \frac{1}{2}\ln\frac{\prod_{i=1}^{d}(\lambda^{(t)}_i + \sigma)}{\sigma^{d}}$, which immediately completes the proof.
\end{proof}

\subsection{Proof of Theorem~\ref{thm:upper_bound_personalized}}
Before the proof of Thm.~\ref{thm:upper_bound_personalized}, we have the following lemma.

\begin{lemma}\label{lemma:reverse_matrix}
For any $d$-dimensional semi-positive definite matrix $\mA$ and $\mB$, we have
\begin{enumerate}[\quad(1)]
      \item $\valpha^{\mathrm{T}}(\mA + \mB)^{-1}\valpha \leq \valpha^{\mathrm{T}}\mA^{-1}\valpha,$
      \item $\valpha^{\mathrm{T}}(\mA + \mB)^{-1}\valpha \leq \valpha^{\mathrm{T}}\mB^{-1}\valpha,$
    \end{enumerate}
where $\valpha$ is a $d$-dimensional vector.
\end{lemma}

\begin{proof}
  Here we only prove inequality (1), and the proof for inequality (2) is the same as inequality (1). First, we have
  \begin{equation}\label{eq:unfold_reverse_matrix}
    (\mA + \mB)^{-1} = \mA^{-1} - \mA^{-1}(\mA^{-1}+\mB^{-1})^{-1}\mA^{-1},
  \end{equation}
  hence, we can get following result.
  $$\valpha^{\mathrm{T}}(\mA + \mB)^{-1}\valpha = \valpha^{\mathrm{T}}(\mA^{-1} -\mA^{-1}(\mA^{-1}+\mB^{-1})^{-1}\mA^{-1})\valpha. $$
  Then we need to prove
  \begin{equation}\label{eq:sum_of_reverse_matrix}
    \valpha^{\mathrm{T}}\mA^{-1}(\mA^{-1}+\mB^{-1})^{-1}\mA^{-1}\valpha \geq 0.
  \end{equation}

  As $\mA^{-1}$ is a symmetric matrix, we can rewrite Eq.~(\ref{eq:sum_of_reverse_matrix}) as
  \begin{equation}\label{eq:sum_of_reverse_matrix_2}
    \tilde{\valpha}^{\mathrm{T}}(\mA^{-1}+\mB^{-1})^{-1}\tilde{\valpha} \geq 0,
  \end{equation}
  where $\tilde{\valpha}=\mA^{-1}\valpha$. As $\mA$ and $\mB$ are semi-positive definite, we have $(\mA^{-1}+\mB^{-1})^{-1}$ is a semi-positive definite matrix either, hence inequality (\ref{eq:sum_of_reverse_matrix_2}) holds, which concludes the proof.
\end{proof}

Then with Lemma~\ref{lemma:reverse_matrix}, we can prove Thm.~\ref{thm:upper_bound_personalized}.

\begin{proof}
  As we can rewrite $\mSigma_i$, $i\in\{0,\cdots, d-1\}$, as
  \begin{align}\label{eq:matrix_decomp_2}
   \mSigma_{i} &= \left( \begin{array}{cc}
                              \mSigma_{i-1} &  \vrho_{i}\\
                                 \vrho_{i}^\mathrm{T} &   c_{i,i}
                            \end{array}\right) \\
                            &=\left( \begin{array}{cc}
                                                           \mI_{i-1} & \mathbf{0} \notag \\
                                                           \vrho_{i}^\mathrm{T}  \mSigma_{i-1}^{-1} & 1
                                                         \end{array} \right)\left( \begin{array}{cc}
                                                                                     \mSigma_{i-1} & \vrho_{i} \\
                                                                                     \mathbf{0} & c_{i,i}-\vrho_{i}^\mathrm{T}\mSigma_{i-1}^{-1}\vrho_{i}
                                                                                   \end{array}\right),
   \end{align}
   Hence, the determinant is
   \begin{equation}\label{eq:single_dim_decompose}
     \det\mSigma_i = (c_{i,i}-\vrho_{i}^\mathrm{T}\mSigma_{i-1}^{-1}\vrho_{i})\det \mSigma_{i-1}.
   \end{equation}
   Based on this result, with substitution $\mSigma_i+\vbeta_{i}$ for $\mSigma_i$, we have
   \begin{align}\label{eq:decompose_of_Matrix}
     &\det(\mSigma_i+\vbeta_i) = v_i \cdot \det (\mSigma_{i-1}+\diag(\vbeta_{i-1})), \\
     &v_i = c_{i,i}+\beta_i-\vrho_{i}^\mathrm{T}(\mSigma_{i-1}+\diag(\vbeta_{i-1}))^{-1}\vrho_{i},
   \end{align}
   where $\vbeta_i = (\beta_0, \cdots, \beta_{i-1})$. Then as $(\mSigma_{i-1}+\diag(\vbeta_{i-1}))^{-1}$ is a semi-positive definite matrix, we have
   \begin{equation}\label{eq:diagnal_upper_bound}
     v_i \leq c_{i,i}+\beta_i.
   \end{equation}

   Furthermore, we can rewrite $v_i$ as   \begin{align*}
     v_i & = &&c_{i,i}+\beta_i-\vrho_{i}^\mathrm{T}\mSigma^{-1}_{i-1}\vrho_{i} \notag \\
         &   &&+ \vrho_{i}^\mathrm{T}\mSigma^{-1}_{i-1}(\mSigma^{-1}_{i-1} + \diag(\vbeta_{i-1})^{-1})^{-1}\mSigma^{-1}_{i-1}\vrho_{i} \notag\\
         & \leq && c_{i,i}+\beta_i-\vrho_{i}^\mathrm{T}\mSigma^{-1}_{i-1}\vrho_{i}+\vrho_{i}^\mathrm{T}\mSigma^{-1}_{i-1}\diag(\vbeta_{i-1})\mSigma^{-1}_{i-1}\vrho_{i},
   \end{align*}
   where the first equality is based on Eq.~(\ref{eq:unfold_reverse_matrix}), and the last inequality depends on Lemma~\ref{lemma:reverse_matrix}.
   Finally, according to Eq.~(\ref{eq:single_dim_decompose}), we have $\ln\det\mSigma_{d-1} = \sum_{i=0}^{d-1} \ln v_i$ by induction.
   Combining above results, we can immediately conclude the proof.
\end{proof}

\subsection{Proof of Theorem~\ref{thm:dp_channel_capacity}}
\begin{proof}
    For the first stage, if we denote the threshold of gradient clipping as $S$ and denote the added noise as $\sigma \cdot \rmI$, we get $\| \vg^{(t)} \|_2 \leq S$, where $\vg^{(t)}=\nabla_{\mW}F(\mW^{(t)}_{i};\; \mD)$ is a random vector as a function of $\mD$. Hence, we have
    \begin{equation}\label{eq:cov_clipping_bound}
    \tr(\mSigma^{(t)}) \leq \tr(\mathE(\vg^{(t)}(\vg^{(t)})^{T}))=\mathE(\| \vg^{(t)} \|^2_2) \leq S^2,
    \end{equation}
    where the first inequality depends on the relationship between covariance matrix and the auto-correlation matrix, i.e.,  $\tr[\mathE(\vg^{(t)}(\vg^{(t)})^{T})-\mSigma^{(t)}] = \tr[\mathE(\vg^{(t)})\mathE(\vg^{(t)})^{T})]\geq 0$.

    Then we have $\sum_{i=1}^{d}\lambda^{(t)}_i=\tr(\mSigma^{(t)})\leq S^2$. Based on AM-GM inequality and the fact $\lambda^{(t)}_i\geq 0$, the channel capacity derived by Eq.~(\ref{eq:info_func})
    has an upper bound as
    $$\setlength{\abovedisplayskip}{5pt}
  \setlength{\belowdisplayskip}{5pt}
  f^{(t)}(\sigma) \leq \hat{f}^{(t)}(\sigma)=d\cdot \ln\frac{\sigma+\frac{\sum_{i=1}^{d}\lambda^{(t)}_i}{d}}{\sigma} \leq d\cdot \ln\frac{\sigma+S^2/d}{\sigma}. $$

    Therefore, if we treat $\sigma$ as a constant, and denote $u_{\sigma}(d) := d\cdot \ln\frac{\sigma+S^2/d}{\sigma}, d \geq 0$, we have $f^{(t)}(\sigma)\leq u_{\sigma}(d)$. We observe that $u^{'}_{\sigma}(d) = \ln\frac{\sigma+S^2/d}{\sigma} + \frac{\sigma}{\sigma+S^2/d}-1$. Hence
    $$u{''}_{\sigma}(d)=-\frac{S^2(2\sigma d + S^2)}{d(\sigma d + S^2)^2} \leq 0 \Rightarrow u^{'}_{\sigma}(d)\geq u^{'}_{\sigma}(+\infty)=0,$$
    which implies $u_{\sigma}(d)$ is an increasing function of the dimension $d$. Based on L'Hospital's rule, we can conclude that
    \begin{equation}\label{eq:channel_capacity_inf}
    \setlength{\abovedisplayskip}{5pt}
  \setlength{\belowdisplayskip}{5pt}
      C^{(t)}=f^{(t)}(\sigma)\leq u_{\sigma}(d) \leq u_{\sigma}(+\infty)=\frac{S^2}{\sigma}.
    \end{equation}

    For the second stage, which provides a $(\epsilon, \delta)$-DP, the authors in \cite{DBLP:journals/fttcs/DworkR14} demonstrate that we need to choose $\sigma \geq S^2\cdot \frac{2\log(1.25/\delta)}{\epsilon^2}$, then with a substitution in Eq.~(\ref{eq:channel_capacity_inf}), we have $C^{(t)}\leq \frac{\epsilon^2}{2\log(1.25/\delta)}$. Hence, the channel capacity of $(\epsilon, \delta)$-DP in the FL scenario is upper bounded by $\frac{\epsilon^2}{2\log(1.25/\delta)}$.

    Finally, as illustrated in \cite{DBLP:journals/fttcs/DworkR14}, for mini-batch SGD in DP, i.e., the batch size $B>1$, the mechanism becomes
    $ \vg = \frac{1}{B}(\sum_{i=1}^{B}\bar{\vg}_i + \vxi)$,
    let $\tilde{\vg} = \frac{1}{B}\sum_{i=1}^{B}\bar{\vg}_i$ and $\tilde{\vxi} = \frac{1}{B}\vxi$, we have $\mSigma_{\tilde{\vg}} = \frac{1}{B}\mSigma_{\vg_i}$ and $\mSigma_{\tilde{\vxi}} = \frac{1}{B^2}\mSigma_{\vxi}$. By substitution them into Eq.~(\ref{eq:cov_clipping_bound}), we can immediately conclude the proof.
\end{proof}


\begin{figure*}[htb]
\begin{center}
\begin{minipage}[b]{0.9\textwidth}
\includegraphics[width=\textwidth]{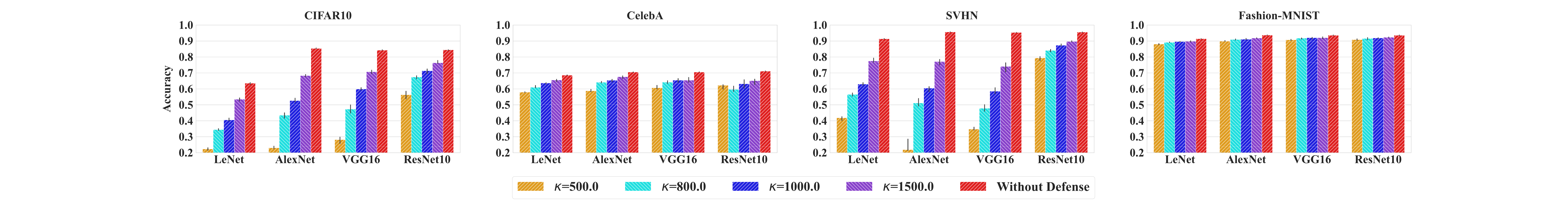}
\vskip -0.1in
\end{minipage}
\vskip -0.1in
\caption{The effect of channel capacities ($\kappa$) for model accuracy. (White Channel)}\label{fig:ModelACC_ChannelCapacity_WhiteChannel}
\begin{minipage}[b]{0.9\textwidth}
\includegraphics[width=\textwidth]{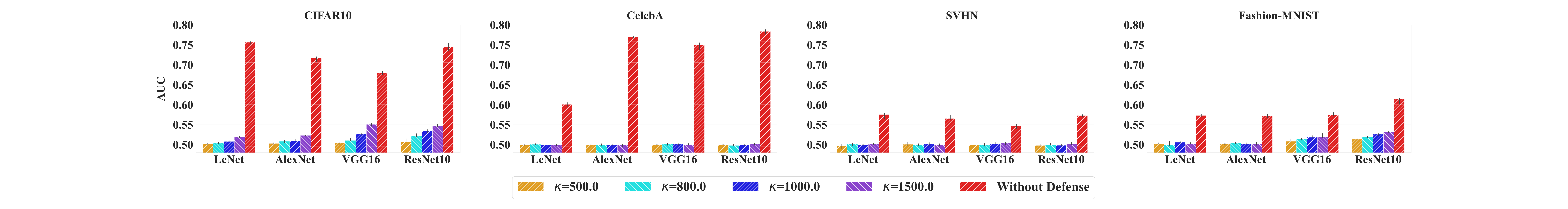}
\vskip -0.1in
\end{minipage}
\vskip -0.1in
\caption{The effect of channel capacities ($\kappa$) for MIA. (White Channel)}\label{fig:MemberInf_ChannelCapacity_WhiteChannel}
\begin{minipage}[b]{0.9\textwidth}
\includegraphics[width=\textwidth]{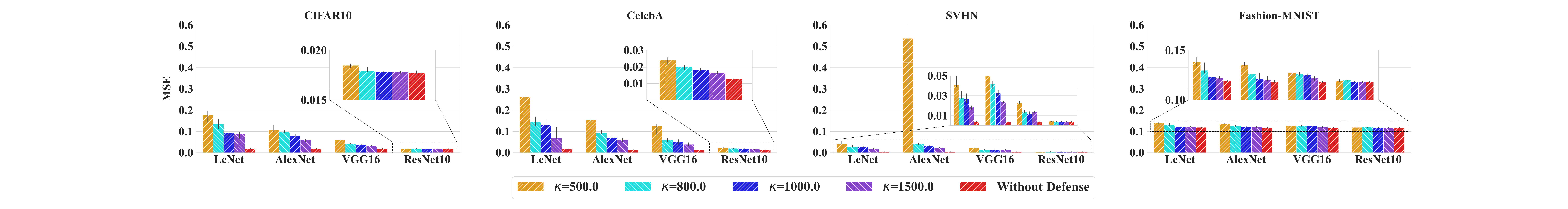}
\vskip -0.1in
\end{minipage}
\vskip -0.1in
\caption{The effect of channel capacities ($\kappa$) for model inversion attack. (White Channel)}\label{fig:ModelInv_ChannelCapacity_WhiteChannel}
\vskip -0.2in
\end{center}
\end{figure*}

\section{\qi{Validating the theories of existing methods}}\label{sec:exp_validating}

\noindent{\bfseries The impact of batch size for DP and utilizing large batch size in defending against DRA.} \qi{To improve the existing defense algorithm for defending against DRA, we design experiments to validate the roles of $B$ by gradient inversion attack on CIFAR-10. Specifically, to keep identical reconstruction difficulties for different $B$, we only reconstruct one image from the gradient as the mean restoration $\mathbb{E}[\hat{\mD}]$. Then we calculate the MSE between $\mathbb{E}[\hat{\mD}]$ and the mean value of local dataset $\mathbb{E}[\mD]$, i.e., $\|\mathbb{E}[\hat{\mD}]-\mathbb{E}[\mD]\|^2$. Due to the convexity of MSE, the difference between $\hat{\mD}$ and $\mD$ is lower bounded, i.e., $\mathbb{E}\|\hat{\mD}-\mD\|^2\geq \|\mathbb{E}[\hat{\mD}]-\mathbb{E}[\mD]\|^2$. Moreover, we utilize clipping bound $S=1.0$ and the noise multiplier $\sigma=1.3$ for the DP training. The results are displayed in Fig.~\ref{fig:batch_size_exp}. If we utilize DP for privacy protection, the performance for defending against DRA consistently decreases with an increasing $B$. This phenomenon is consistent with Thm.~\ref{thm:dp_channel_capacity}, indicating that the small $B$ is beneficial for DP in defending against DRA.}

\qi{Additionally, when we do not apply any defense technique for privacy, the defense ability increases according to the increasing $B$. This phenomenon is consistent with Eq.~(\ref{thm:large_batch_size}) due to the intrinsic noise in the collected data. Moreover, the performance improvement becomes more significant when we add a constant noise to the parameters. The conclusion is that when we add a constant noise to the parameter, a larger $B$ leads to a stronger ability to defend against DRA.}

\noindent{\bfseries The impact of eigenvalues in compression.} \qi{We also validate our theories for compression on CIFAR-10 datasets. In these experiments, we utilize the gradient inversion attack to reconstruct data from the gradients. For the compression, we set the eigenvector to be $\mathbf{0}$ to reduce the information in the corresponding dimension. Then we map the training data to the compressed eigenspace. For comparison, we only compress one dimension according to the eigenvalue. The results are displayed in Fig.~\ref{fig:Compression_Eigen}. Specifically, the eigenvalues $\{\lambda_{max}, \lambda_1, \lambda_{10}\}$ are as follows: \{222.21, 85.08, 9.91\} for CIFAR10, \{346.21, 71.01, 12.82\} for CelebA, \{313.0, 28.04, 4.67\} for SVHN, and \{101.88, 61.80, 3.21\} for Fashion-MNIST. The experiment results demonstrate that the compression on the dimension with a larger eigenvalue leads to a stronger defensive ability to defend against DRA, which is consistent with our theories.}


\begin{figure*}[htb]
\begin{center}
\begin{minipage}[b]{0.9\textwidth}
\includegraphics[width=\textwidth]{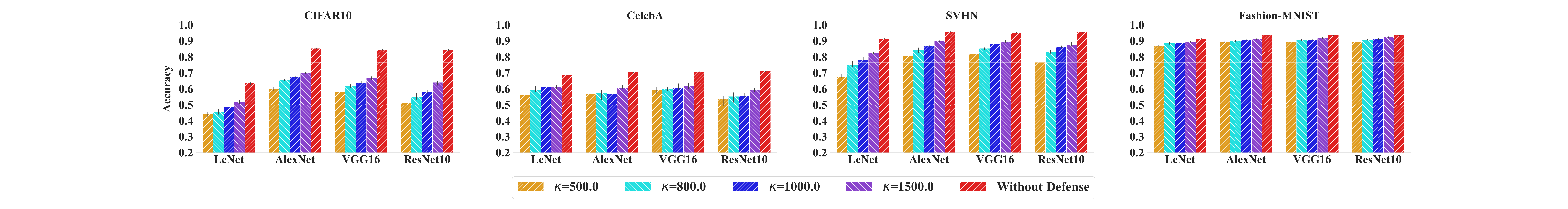}
\vskip -0.1in
\end{minipage}
\vskip -0.1in
\caption{The effect of channel capacities ($\kappa$) for model accuracy. (Personalized Channel)}\label{fig:ModelACC_ChannelCapacity_PersonalizedChannel}
\begin{minipage}[b]{0.9\textwidth}
\includegraphics[width=\textwidth]{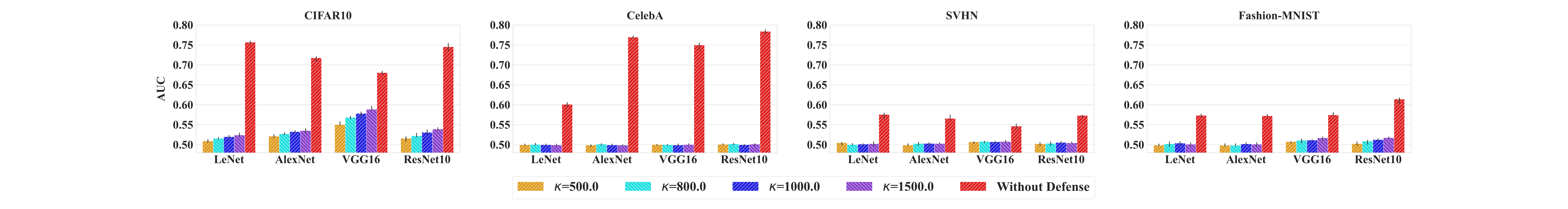}
\vskip -0.1in
\end{minipage}
\vskip -0.1in
\caption{The effect of channel capacities ($\kappa$) for MIA. (Personalized Channel)}\label{fig:MemberInf_ChannelCapacity_PersonalizedChannel}
\begin{minipage}[b]{0.9\textwidth}
\includegraphics[width=\textwidth]{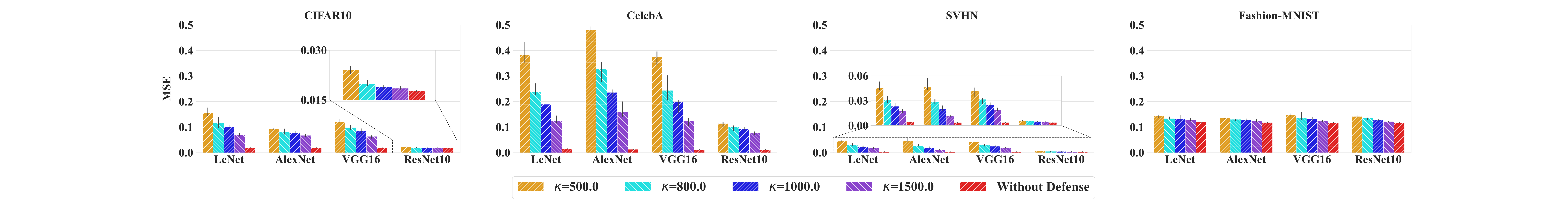}
\vskip -0.1in
\end{minipage}
\vskip -0.1in
\caption{The effect of channel capacities ($\kappa$) for model inversion attack. (Personalized Channel)}\label{fig:ModelInv_ChannelCapacity_PersonalizedChannel}
\vskip -0.2in
\end{center}
\end{figure*}

\section{Channel Capacities of Existing Methods}\label{subsec:more_existing_method}
\noindent{\bfseries Compression.} As for gradient compression, it is easy to understand that the compression mechanism reduces the information contained in parameters, and we can also formalize this intuition with our channel model.

On the one hand, in the eigenspace of $\mSigma^{(t)}$, the channel capacity is $C^{(t)}=\frac{1}{2}\sum_{i=1}^{d}\ln\frac{\lambda^{(t)}_i + \sigma}{\sigma}$, and due to the non-negativity of $\lambda^{(t)}_i$, $\ln\frac{\lambda^{(t)}_i + \sigma}{\sigma}\geq 0$, $i\in\{1,\cdots,d\}$, represents the upper bound of information contained in the different independent dimensions. As illustrated in Fig.~\ref{fig:Rel_Compre_Noise}, if we set a positive $\lambda^{(t)}_i$ to be $0$, then the corresponding $C^{(t)}$ decreases.

On the other hand, the gradient compression is directly applied for the gradient $\vg$ \cite{DBLP:conf/iclr/LinHM0D18,DBLP:conf/iclr/TsuzukuIA18}, hence we need to analyze $\mSigma^{(t)}$ itself. Specifically, in this situation, the information contained in different dimensions is not independent, so we cannot directly analyze the information contained in each dimension. Without loss of generality, we can rewrite $\mSigma^{(t)}$ as

  \begin{align}\label{matrix_decomp}
  \mSigma^{(t)} = &\left( \begin{array}{cc}
                              \mP_{d-k} & \mR \\
                              \mR^{\mathrm{T}}    & \mQ_{k}
                            \end{array}\right) \\ = &\left( \begin{array}{cc}
                                                           \mI_{d-k} & \mathbf{0} \notag \\
                                                           \mR^{\mathrm{T}}  \mP_{d-k}^{-1} & \mI_{k}
                                                         \end{array} \right)\left( \begin{array}{cc}
                                                                                     \mP_{d-k} & \mR \\
                                                                                     \mathbf{0} & \mQ_{k}-\mR^{\mathrm{T}}\mP_{d-k}^{-1}\mR
                                                                                   \end{array}\right),
  \end{align}
where $\mP_{d-k}\in \mathbb{R}^{(d-k)\times (d-k)}$ represents the covariance matrix of the uncompressed dimension, $\mQ_{k}\in \mathbb{R}^{k\times k}$ represents the covariance matrix of the compressed dimension, and $\mR \in \mathbb{R}^{(d-k)\times k}$ represents the mutual covariance matrix between the compressed dimension and the uncompressed dimension. Thus, we have
\begin{equation}\label{eq:model_compression}
  \det(\mSigma^{(t)}) = \det(\mP_{d-k})\cdot\det(\mQ_{k}-\mR^{\mathrm{T}}\mP_{d-k}^{-1}\mR).
\end{equation}
Combining Eq.~(\ref{eq:model_compression}) with Thm.~\ref{thm:channel_capacity}, we can conclude that the decrease of channel capacity caused by compression can be formalized as
$$\Delta C^{(t)}=\frac{1}{2}\ln\frac{\det(\mQ_{k}+\sigma\mI_{k}-\mR^{\mathrm{T}}(\mP_{d-k}+\sigma\mI_{d-k})^{-1}\mR)}{\sigma^k}.$$

As the covariance matrix is positive semi-definite, we conclude that $\Delta C^{(t)} \geq 0$.


\begin{figure*}
\begin{center}
\begin{minipage}[b]{0.9\textwidth}
\vskip -0.3in
\includegraphics[width=\textwidth]{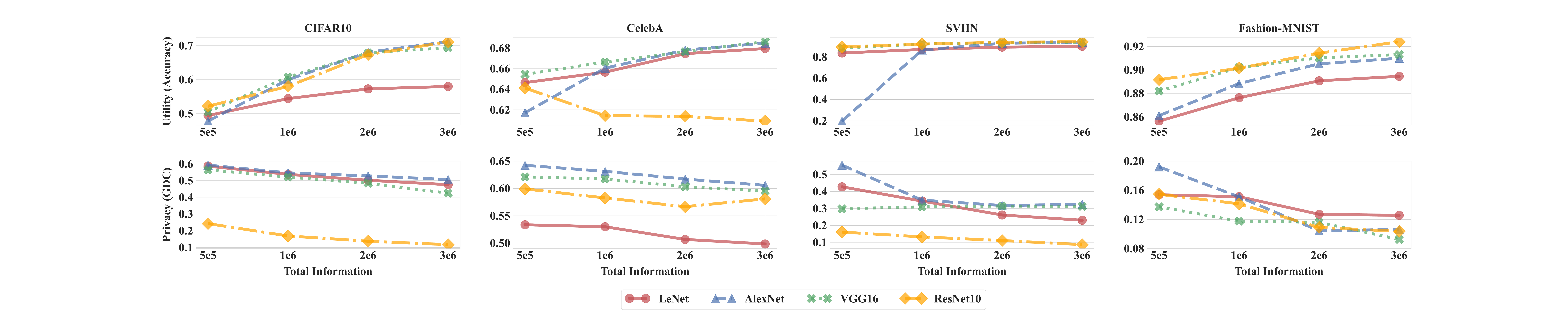}
\vskip -0.1in
\end{minipage}
\vskip -0.1in
\caption{Utility privacy tradeoff according to different optimization number ($n$) when $\kappa=300$.}\label{fig:Utility_Privacy_Tradeoff_Number}
\end{center}
\end{figure*}

\section{Additional Experiments}\label{sec:appendix_exp}

\begin{figure*}[htb]
    \begin{center}
\begin{minipage}[b]{0.9\textwidth}
        \includegraphics[width=\textwidth]{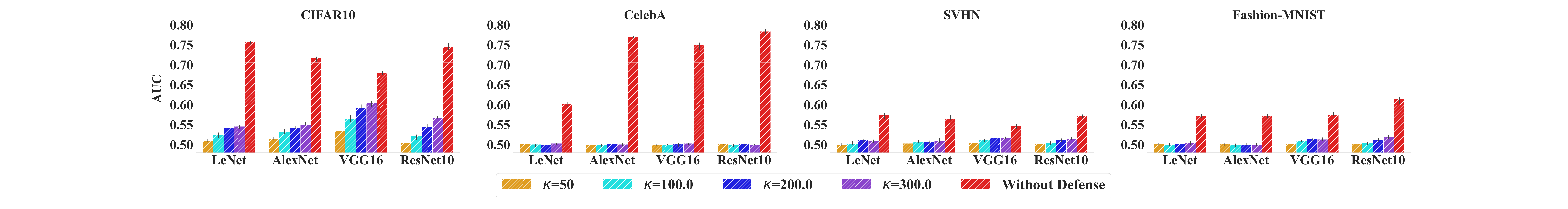}
        \vskip -0.1in
        \end{minipage}
        \vskip -0.1in
        \caption{The effect of channel capacities ($\kappa$) for MIA. (Natural Channel)}\label{fig:MemberInf_ChannelCapacity}
    \begin{minipage}[b]{0.9\textwidth}
        \includegraphics[width=\textwidth]{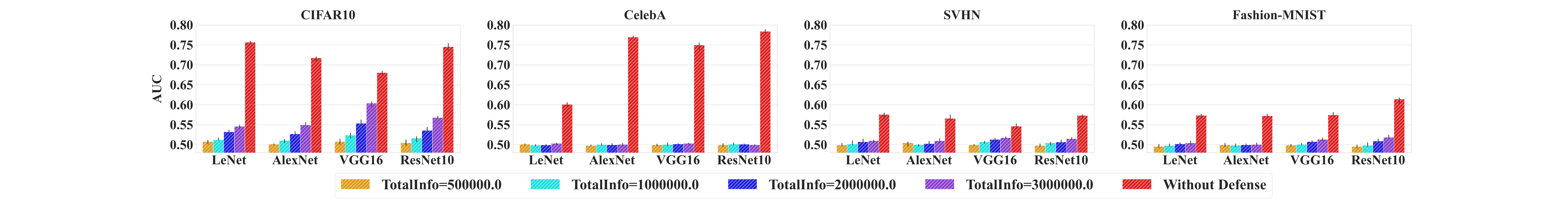}
        \vskip -0.1in
        \end{minipage}
        \vskip -0.1in
        \caption{The effect of optimization number ($n$) for MIA when $\kappa=300$. (Natural Channel)}\label{fig:MemberInf_Number}
    \end{center}
\end{figure*}



\begin{figure*}[htb]
    \begin{center}
        \begin{minipage}[b]{0.9\textwidth}
        \includegraphics[width=\textwidth]{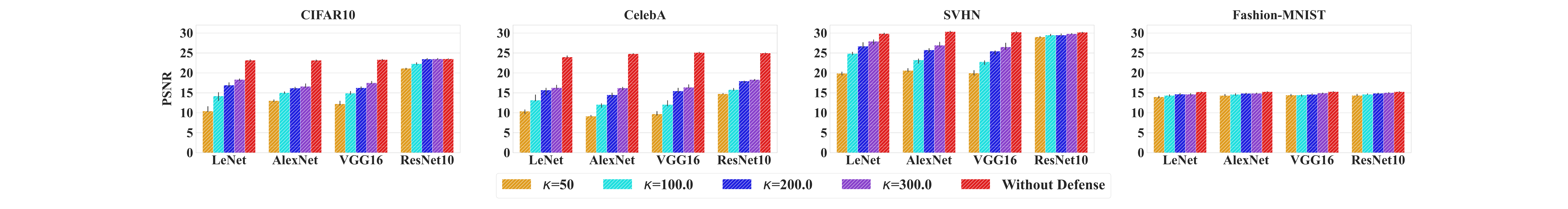}
        \end{minipage}
        \caption{Additional experiments of PSNR for the model inversion attacks. (Natural Channel)}\label{fig:modinv_psnr}
        \begin{minipage}[b]{0.9\textwidth}
        \includegraphics[width=\textwidth]{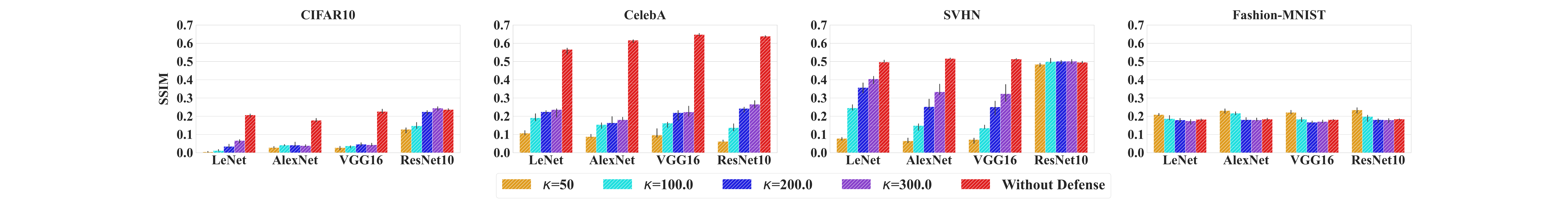}
        \end{minipage}
        \caption{Additional experiments of SSIM for the model inversion attacks. (Natural Channel)}\label{fig:modinv_SSIM}
        \begin{minipage}[b]{0.9\textwidth}
        \includegraphics[width=\textwidth]{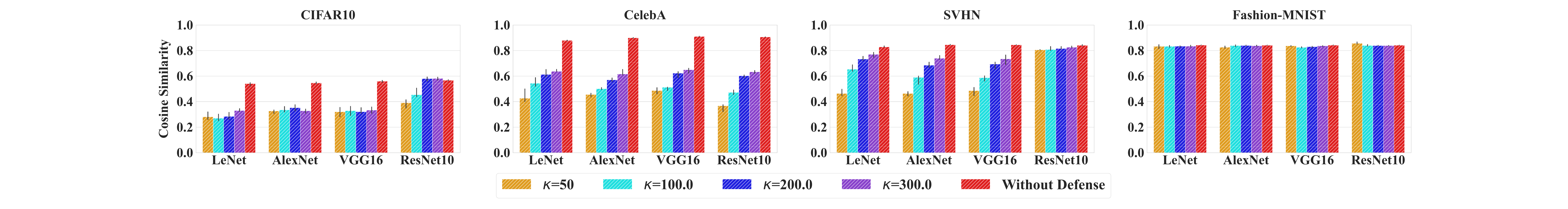}
        \end{minipage}
        \caption{Additional experiments of Cosine Similarity for the model inversion attacks. (Natural Channel)}\label{fig:modinv_cosine_sim}
    \end{center}
\end{figure*}

\begin{figure*}[htb]
    \begin{center}
\begin{minipage}[b]{0.9\textwidth}
        \includegraphics[width=\textwidth]{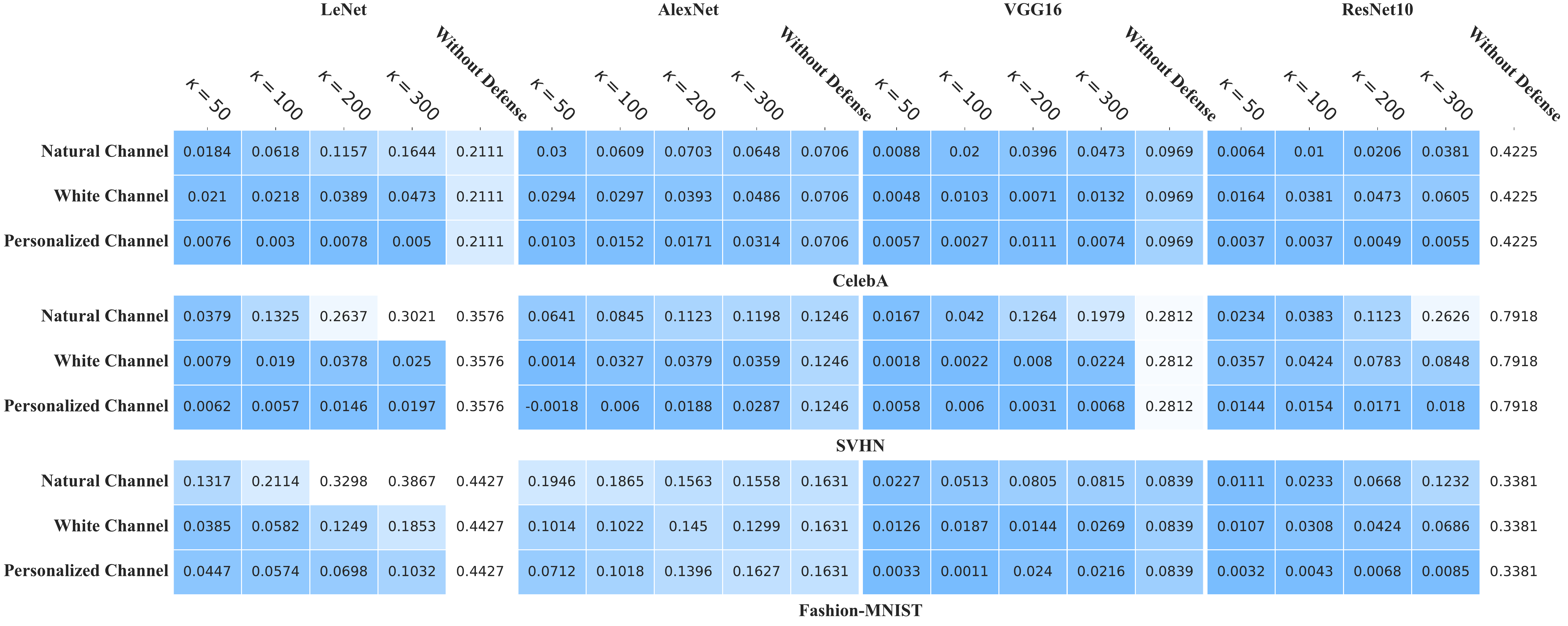}
        \vskip -0.1in
        \end{minipage}
        \vskip -0.1in
        \caption{Additional heatmaps of MSE for the gradient inversion attacks.}\label{fig:Heat_map_MSE_appendix}
\vskip -0.2in
\end{center}
\end{figure*}

\begin{figure*}[htb]
    \begin{center}
        \begin{minipage}[b]{0.9\textwidth}
        \includegraphics[width=\textwidth]{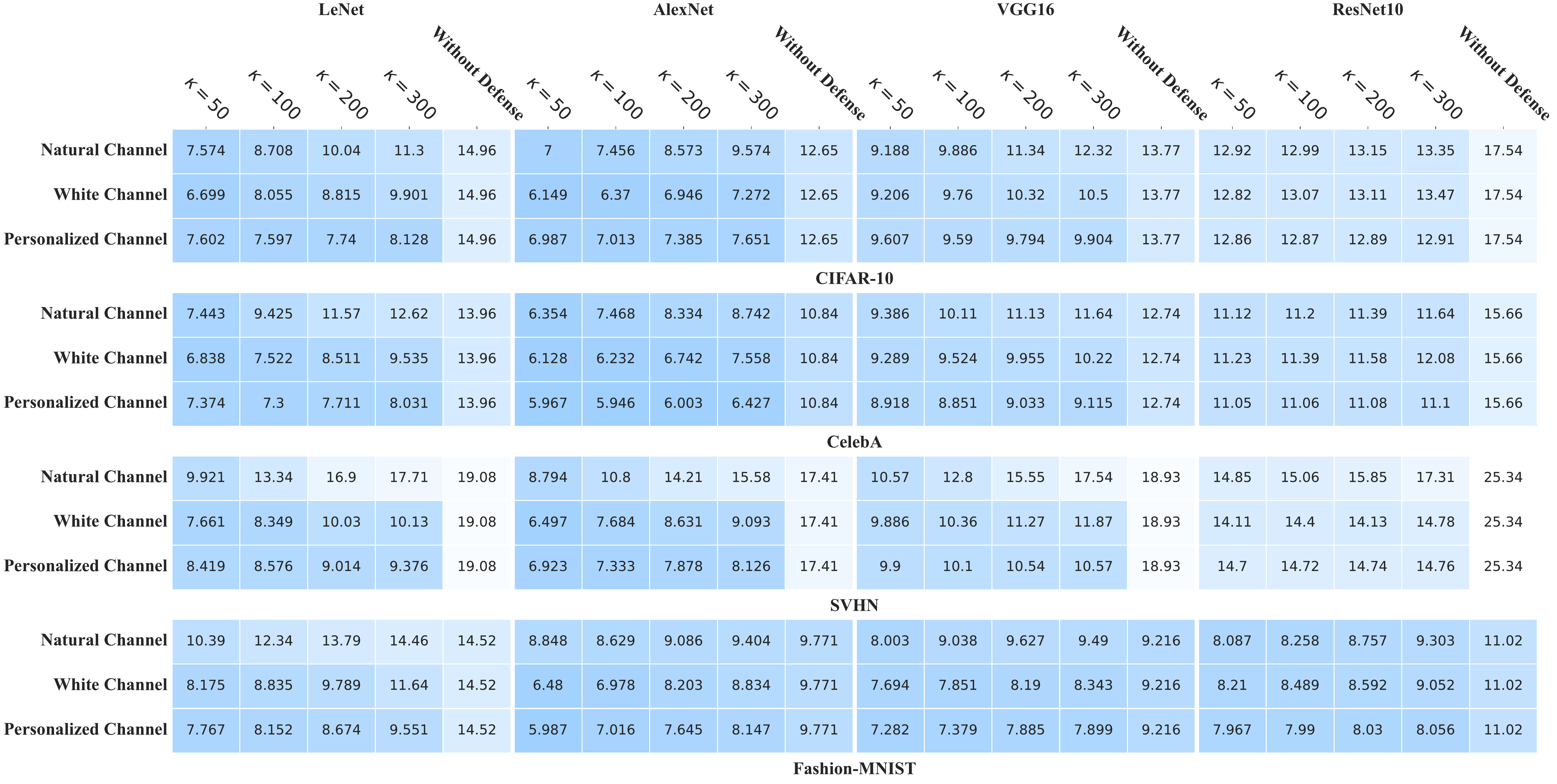}
        \end{minipage}
        \caption{Additional heatmaps of PSNR for the gradient inversion attacks.}\label{fig:heat_map_psnr}
    \end{center}
\end{figure*}

\begin{figure*}[htb]
    \begin{center}
        \begin{minipage}[b]{0.9\textwidth}
        \includegraphics[width=\textwidth]{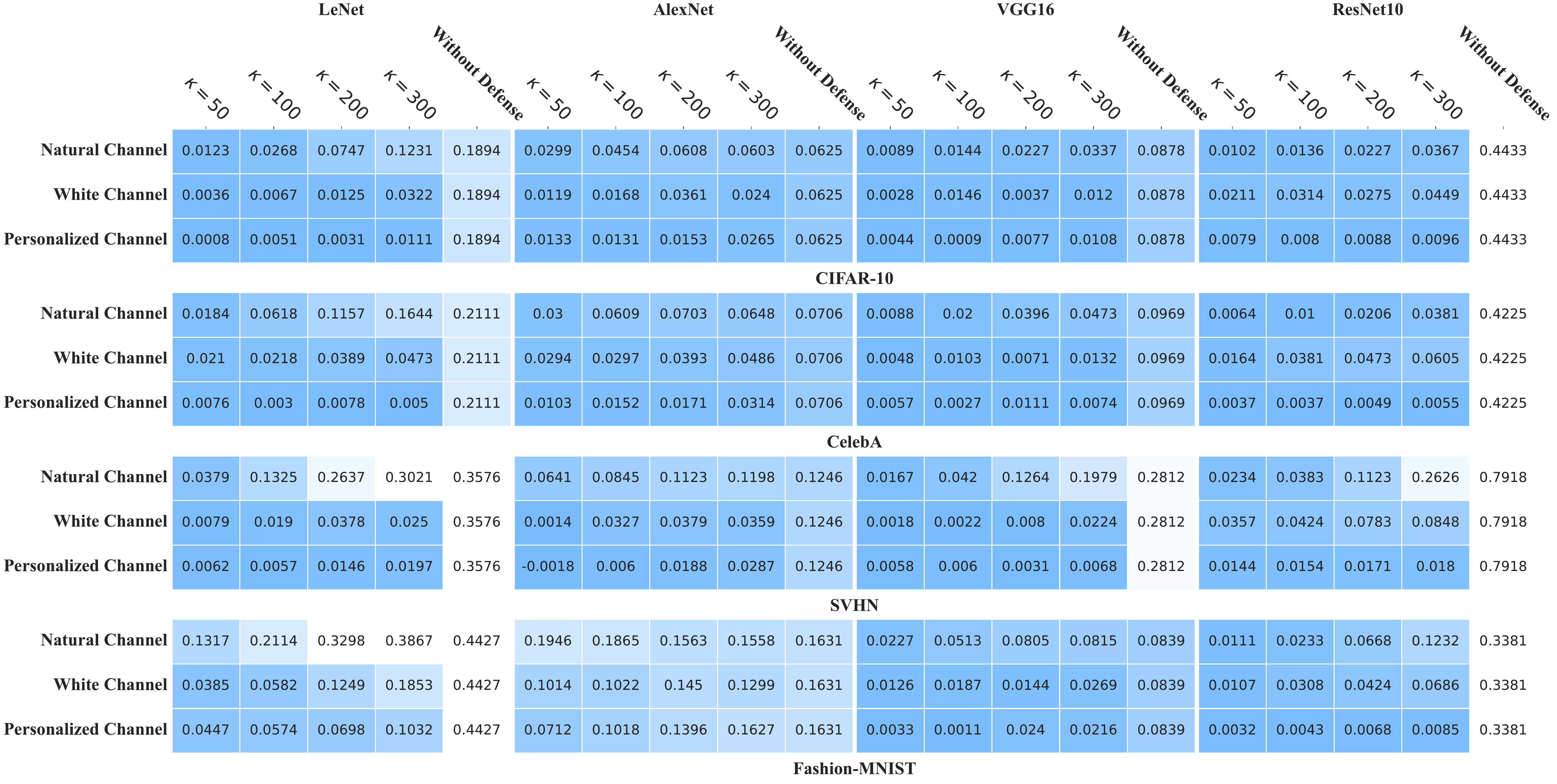}
        \end{minipage}
        \caption{Additional heatmaps of SSIM for the gradient inversion attacks.}\label{fig:heatmap_SSIM}
    \end{center}
\end{figure*}

\begin{figure*}
    \begin{center}
                \begin{minipage}[b]{0.9\textwidth}
        \includegraphics[width=\textwidth]{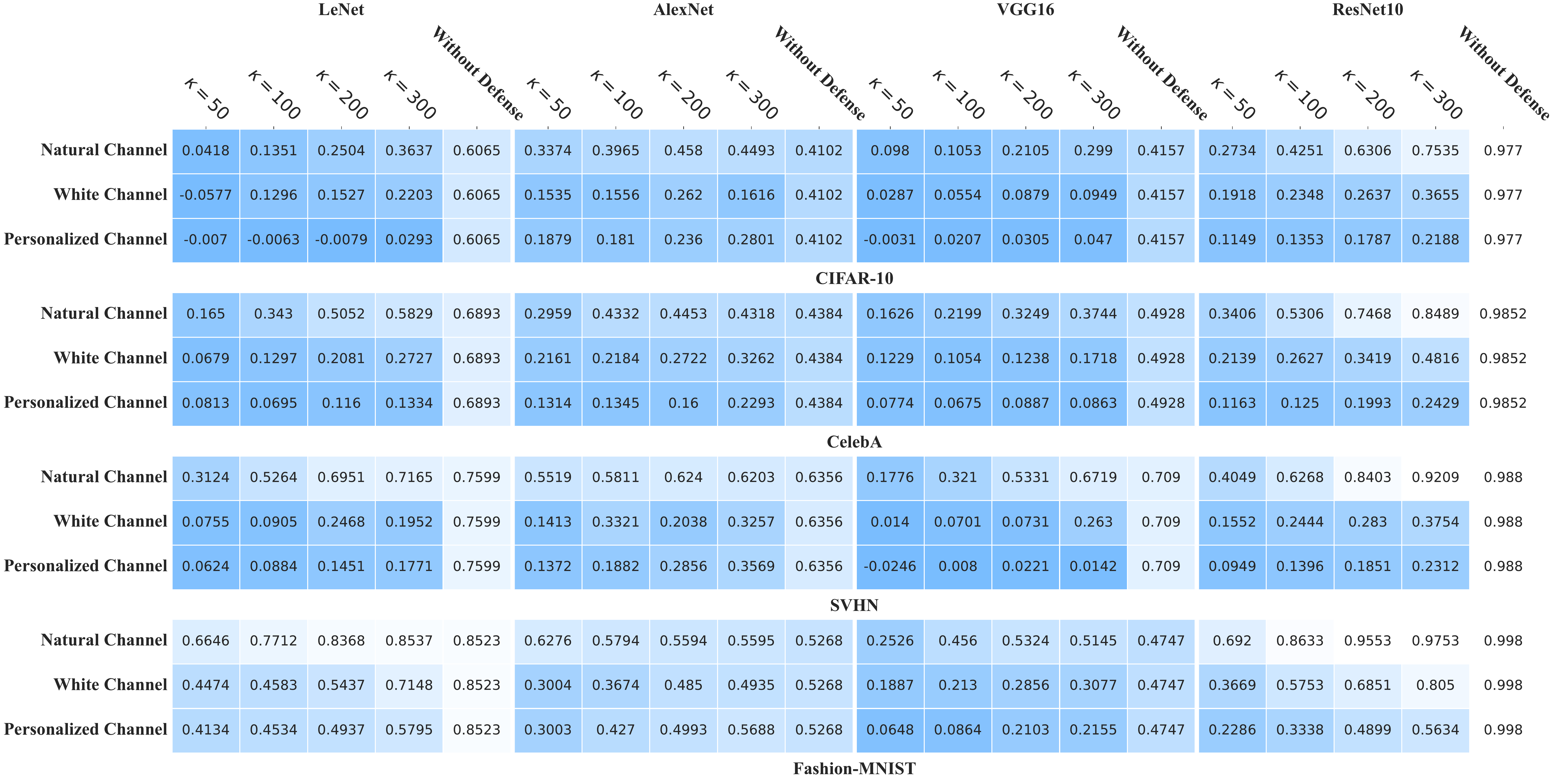}
        \vskip -0.1in
        \end{minipage}
        \vskip -0.1in
        \caption{Additional heatmaps of Cosine Similarity for the gradient inversion attacks.}\label{fig:heatmap_cosine_sim}
    \end{center}
\end{figure*}

\qi{In this section, we display additional experimental results to evaluate our techniques. The results are summarized as follows:}

\begin{itemize}
    \item \qi{{\bfseries The results for White Channel and Personalized Channel.} Fig.~\ref{fig:ModelACC_ChannelCapacity_WhiteChannel} and Fig.~\ref{fig:ModelInv_ChannelCapacity_WhiteChannel} display the model accuracy and the defense ability to defend against DRA for White Channel. Meanwhile, Fig.~\ref{fig:ModelACC_ChannelCapacity_PersonalizedChannel} and Fig.~\ref{fig:ModelInv_ChannelCapacity_PersonalizedChannel} display the corresponding results for Personalized Channel. Specifically, these results are similar to the Natural Channel, when we increase the channel capacity, the model accuracy increases, while the defense ability decreases.}
    \item \qi{{\bfseries The defense ability for defending against MIA.} The results in Fig.~\ref{fig:MemberInf_ChannelCapacity_WhiteChannel}, Fig.~\ref{fig:MemberInf_ChannelCapacity_PersonalizedChannel}, and Fig.~\ref{fig:MemberInf_ChannelCapacity} indicate that reducing channel capacity can also defend against MIA. Specifically, a smaller AUC means a stronger ability for defending against MIA, which indicates that it is difficult for an attacker to infer the membership property of the given data. When we constrain the transmitted information, the AUC for conducting MIA significantly decreases.
    Moreover, Fig.~\ref{fig:MemberInf_Number} exhibits a similar phenomenon when we restrict the optimization rounds $n$. All of these results demonstrate that constraining transmitted information can also defend against MIA.}
    \item \qi{{\bfseries Utility privacy tradeoff according to optimization rounds $n$.} Fig.~\ref{fig:Utility_Privacy_Tradeoff_Number} displays the utility privacy tradeoff according to $n$. Similarly, we fix the channel capacity $\kappa=300$, and vary the total transmitted information in \{$5 \times 10^5$, $1 \times 10^6$, $2 \times 10^6$, $3 \times 10^6$\} (where $n=\lfloor \frac{TotalInfo}{\kappa} \rfloor$, where $\lfloor \cdot \rfloor$ is the floor function.). Specifically, increasing $n$ leads to more transmitted information, thus results in the increase in utility and the decrease in the ability of privacy protection. These results demonstrate that adjusting $n$ can also influence the tradeoff between the utility and the privacy. }
    \item \qi{{\bfseries More metrics for evaluating the ability to defend against DRA.} For random variables, MSE is $0$ implies the convergence in mean, which further implies convergence in distribution, i.e., the distributions of these variables are identical~\cite{wasserman2004all}. Meanwhile, a smaller MSE indicates the reconstructed data distribution gets closer to the target distribution. Fig.~\ref{fig:Heat_map_MSE_appendix} displays the MSE for CelebA, SVHN, and Fashion-MNIST. The results indicate that when we constrain the transmitted information, the MSE increases, which indicates that it is more difficult for an attacker to conduct DRA attacks.  \\
    In addition to the MSE, we also utilize other metrics such as PSNR, SSIM, and Cosine Similarity to measure the similarity between the reconstructed data and the target data. Among them, PSNR is the most popular metric to measure the quality of the processed image, SSIM is the second most popular metric for this target \cite{PopularMetric}, and Cosine Similarity is the most familiar metric for us to measure the similarity of two vectors. For all of them, the larger value means more similarities in the reconstruction, which indicates the weaker ability to defend against DRA. Fig.~\ref{fig:modinv_psnr} and Fig.~\ref{fig:heat_map_psnr} display the results of PSNR, Fig.~\ref{fig:modinv_SSIM} and Fig.~\ref{fig:heatmap_SSIM} show the reconstruction error by SSIM, Fig.~\ref{fig:modinv_cosine_sim} and Fig.~\ref{fig:heatmap_cosine_sim} utilize Cosine Similarity to measure the quality of reconstructed data. All of these results demonstrate that reducing channel capacity can effectively reduce the quality of reconstruction. In other words, our method can enhance the ability to defend against DRA, which is consistent with Thm.~\ref{mse_lower_bound}. \\
    Particularly, we put emphasis on the metric of MSE. We choose MSE as the main metric for three reasons: first, MSE indicates the convergence of random variables~\cite{wasserman2004all}. That is, if the MSE of two random variables is $0$, we conclude that they have identical distributions, which means a perfect reconstruction. Second, Thm.~\ref{mse_lower_bound} indicates that restricting transmitted information can provide the theoretical lower bound for MSE. Finally, MSE is a general metric for random variables.}

\end{itemize}

\section{Discussion of the Channels}\label{sec:advantage_and_disadvantage}

\qi{In this section, we summarize the advantages and disadvantages of the Natural Channel, the While Channel, and the Personalized Channel as follows:}

\begin{itemize}
    \item {\bfseries Natural Channel}. \qi{It preserves the relative importance of different attributes, hence it maintains the information contained in the dataset, which is beneficial for the model utility in FL. However, Natural Channel cannot incorporate prior knowledge into the protecting process, hence we cannot adjust the algorithm according to the requirement when channel capacity is decided.}
    \item {\bfseries White Channel.} \qi{It provides a much stronger defense ability to defend against DRA when there is no prior knowledge. It treats the importance of each attribute to be equal, which adds more noise to the dimension with more information to achieve such stronger protection. However, the computation complexity of the White Channel is higher than the Natural Channel because of the additional matrix multiplication.}
    \item {\bfseries Personalized Channel.} \qi{It allows us to leverage prior knowledge of the attribute importance to the protection process, hence we can perform the better utility-privacy tradeoff. Whereas, in order to reduce the computation complexity, the Personalized Channel depends on an upper bound, which weakens this benefit accordingly.}
\end{itemize}

\end{document}